\documentclass[10pt,twocolumn,letterpaper]{article}

\usepackage[pagenumbers]{cvpr} 

\definecolor{cvprblue}{rgb}{0.21,0.49,0.74}
\usepackage[pagebackref,breaklinks,colorlinks,allcolors=cvprblue]{hyperref}

\usepackage[most]{tcolorbox}

 \usepackage{multirow}
 \usepackage[table]{xcolor}

\usepackage{algorithm}
\usepackage{algorithmic}
\usepackage{amsthm}
\usepackage{amsmath}
\usepackage{comment}
\newtheorem{assumption}{Assumption}
\newtheorem{theorem}{Theorem}

\newtheorem{lemma}{Lemma}

\newtheorem{definition}{Definition}

\usepackage{colortbl}
\usepackage{xcolor}
\definecolor{lightpink}{RGB}{255, 230, 235}
\usepackage{bm}

\definecolor{LightRed}{RGB}{255,182,193}
\definecolor{LightBlue}{RGB}{173,216,230}
\usepackage{pifont}

\def\paperID{*****} 
\def\confName{CVPR}
\def\confYear{2026}

\title{DP-FedAdamW: An Efficient Optimizer for Differentially Private Federated Large Models}

\author{
Jin Liu \quad Yinbin Miao \quad Ning Xi\\
School of Cyber Engineering, Xidian University, Xi'an, China\\
{\tt\small \{jinliu9787,ybmiao,nxi\}@xidian.edu.cn}
\and
Junkang Liu\\
College of Intelligence and Computing, Tianjin University, Tianjin, China\\
{\tt\small junkangliukk@gmail.com}
}

\begin{document}
\maketitle
\begin{abstract}
Balancing convergence efficiency and robustness under Differential Privacy (DP) is a central challenge in Federated Learning (FL). While AdamW accelerates training and fine-tuning in large-scale models, we find that directly applying it to Differentially Private FL (DPFL) suffers from three major issues: (i) data heterogeneity and privacy noise jointly amplify the variance of second-moment estimator, (ii) DP perturbations bias the second-moment estimator, and (iii) DP amplify AdamW’s sensitivity to local overfitting, worsening client drift. We propose DP-FedAdamW, the first AdamW-based optimizer for DPFL. It restores AdamW under DP by stabilizing second-moment variance, removing DP-induced bias, and aligning local updates to the global descent to curb client drift.
Theoretically, we establish an unbiased second-moment estimator and prove a linearly accelerated convergence rate without any heterogeneity assumption, while providing tighter $(\varepsilon,\delta)$-DP guarantees.
Our empirical results demonstrate the effectiveness of DP-FedAdamW across language and vision Transformers and ResNet-18. On Tiny-ImageNet (Swin-Base, $\varepsilon=1$), DP-FedAdamW outperforms the state-of-the-art (SOTA) by 5.83\%. The
code is available in Appendix.

\end{abstract}    
\vspace{-2mm}
\section{Introduction}
\label{sec:intro}

\begin{figure}[h]
    \centering
    \includegraphics[width=0.9\linewidth]{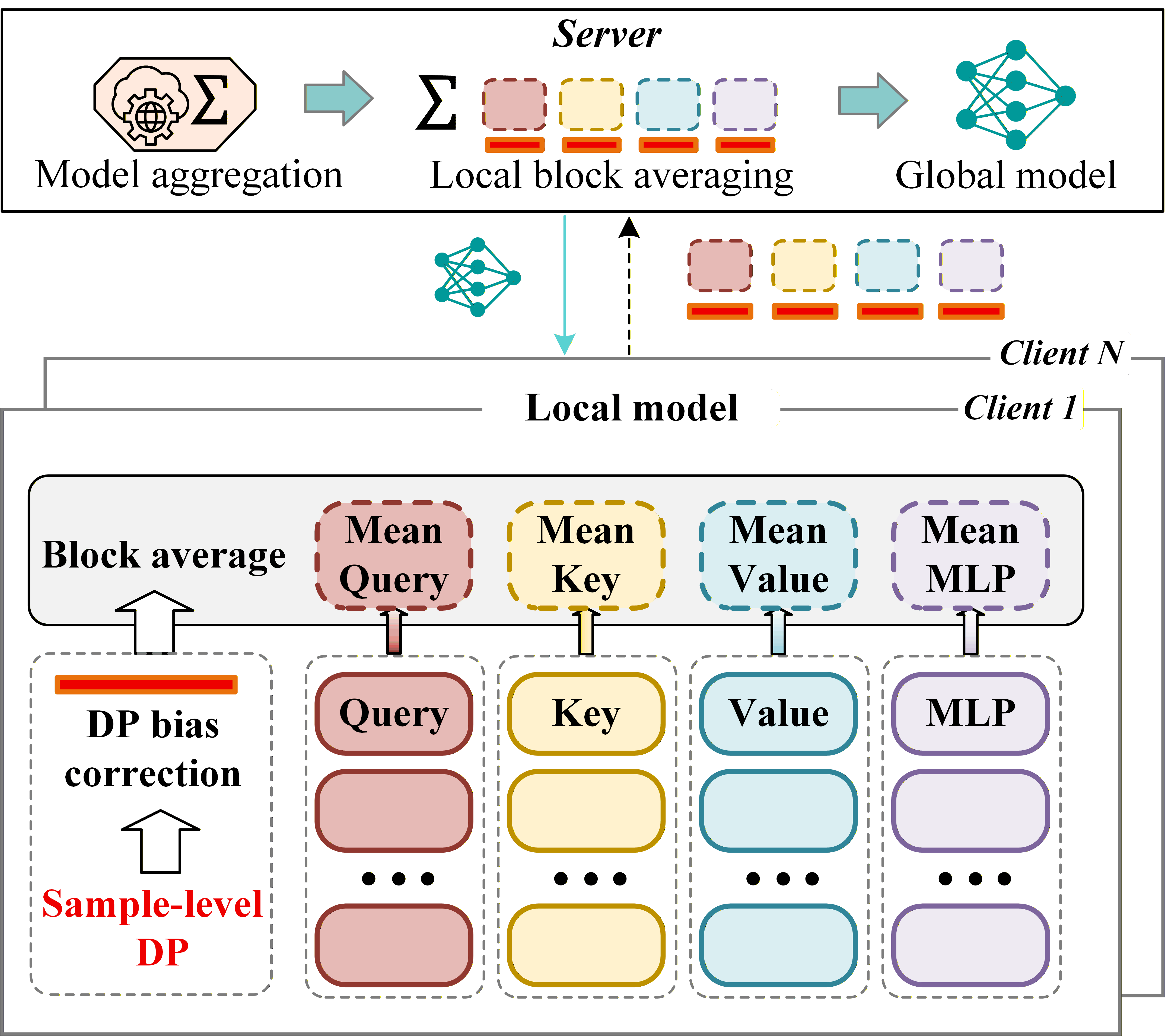}
    \vspace{-2mm}
    \caption{Illustration of our DP-FedAdamW. It aggregates the mean of block-wise Bias-Corrected second moment estimates and performs local–global alignment to stabilize DPFL optimization.}
    \label{figure method}
    \vspace{-2mm}
\end{figure}
\vspace{-2mm}

\begin{figure*}[h]
    \centering
    \includegraphics[width=0.99\linewidth]{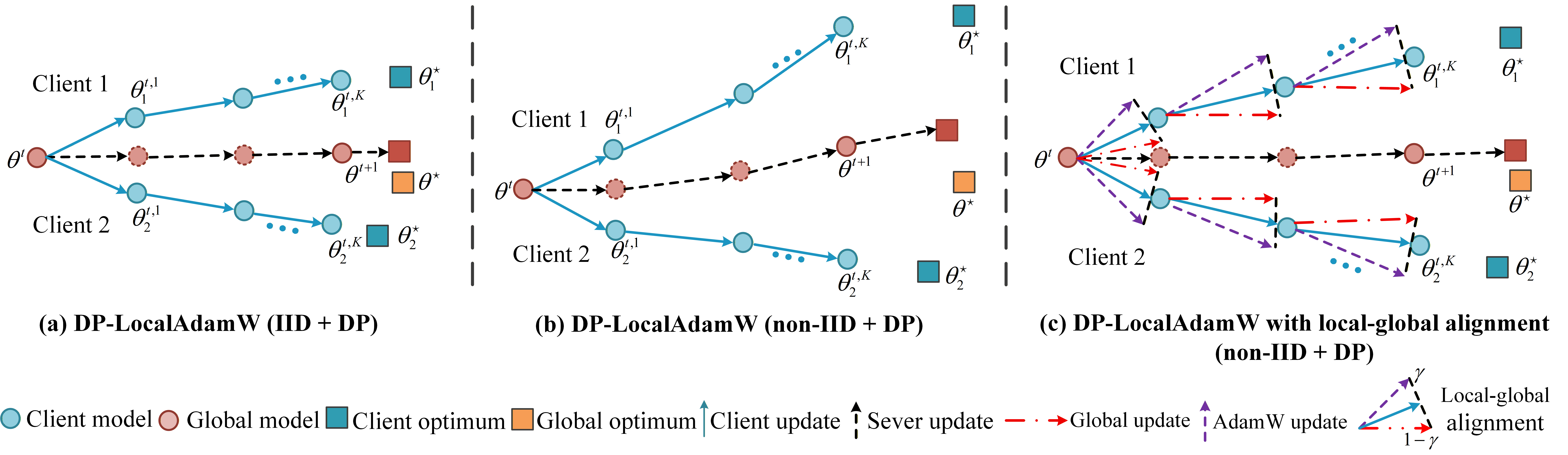}
    \vspace{-2mm}
    \caption{An illustration of local update in DP-FedAdamW, which corrects client drift caused through global update guidance.}
    \label{figure method}
    \vspace{-2mm}
\end{figure*}

As data volumes surge and privacy concerns grow, centralized learning methods are facing increasing limitations in terms of security and collaboration. Federated Learning (FL)~\cite{mcmahan2017communication} enables decentralized training without sharing raw data, reducing privacy risks~\cite{chen2024fair}. However, FL model updates remain vulnerable to attacks~\cite{xu2024dual,fan2025refiner,xu2025risk}. Differential Privacy (DP)~\cite{dwork2006differential} provides formal protection by adding noise, but this increases gradient variance and hinders convergence, especially under non-IID data~\cite{fu2024differentially}.

The emergence of large-scale models such as Swin Transformer~\cite{liu2021swin}, BERT~\cite{kenton2019bert}, and GPT~\cite{achiam2023gpt} has become dominant in vision and language tasks. These deep, parameter-heavy models are highly sensitive to optimization settings, which not only affects training stability but also impacts overall performance. Most DPFL algorithms are built upon the SGD optimizer and adopt the DPSGD~\cite{abadi2016deep} mechanism for privacy preservation~\cite{geyer2017differentially,noble2022differentially,yang2023privatefl,gu2025dp}. 
While effective for small-scale networks, DPSGD becomes inefficient for large Transformer models as the combination of gradient clipping and injected noise severely distorts gradients, amplifying instability and slowing convergence. 


In the pursuit of more efficient optimizers, adaptive methods like AdamW~\cite{loshchilov2017fixing} have become state-of-the-art (SOTA) in deep learning. By decoupling weight decay from adaptive moment estimation, AdamW provides enhanced stability and generalization, making it a compelling choice for various machine learning tasks. These advantages make AdamW a promising candidate for DPFL, where privacy and robustness are crucial. However, our experiments reveal that AdamW fails to deliver its usual benefits in the DPFL setting. In some cases, its performance is comparable to or even worse than SGD. As illustrated in ~\Cref{tab:resnet18_vitbase}, when AdamW is applied locally without any modifications (referred to as DP-LocalAdamW), its performance on the Swin-Tiny model (CIFAR-100, Dirichlet distribution $\alpha{=}0.1$) is worse than DP-FedAvg-LS. These findings suggest that AdamW struggles with the challenges posed by privacy constraints and decentralized FL.

\textbf{Challenges.} We identify three main challenges that contribute to the unexpected results. \textbf{(i) Second-moment estimator variance amplification.} Non-IID client data and DP noise jointly inflate the variance of the second-moment estimator, leading to unstable adaptive scaling in AdamW. \textbf{(ii) Bias in second-moment estimator.} Gradient clipping and noise injection introduce a systematic bias in the second-moment estimator of AdamW. \textbf{(iii) Client crift.} Under non-IID data, DP clipping and noise amplify AdamW’s sensitivity to local overfitting, further worsening client drift. All challenges are explored in more detail in Section~\ref{sec:challenges}. These findings raise a key question: \textbf{How can AdamW maintain effectiveness in the presence of differential privacy and federated heterogeneity?}

To answer this, we propose DP-FedAdamW, a new differentially private adaptive optimizer that restores the effectiveness of AdamW under non-IID data and privacy noise. DP-FedAdamW (i) aggregates second-moment estimates in a block-wise manner to stabilize variance and improve communication efficiency, (ii) eliminates DP-induced bias through an unbiased second-moment correction, and (iii) enforces local–global update alignment to mitigate client drift, achieving more stable and accurate convergence. 

\noindent \textbf{Our contributions} are summarized as follows:

\begin{itemize}[leftmargin=*]
    \item \textbf{Empirical importance of AdamW and challenges in DPFL.} We empirically demonstrate the effectiveness of AdamW in DPFL. Our analysis reveals three key challenges: high variance in second-moment estimator, bias in second-moment estimator, and high client drift.
    \item \textbf{DP-FedAdamW: A principled adaptive optimizer for DPFL.} We propose DP-FedAdamW, a communication-efficient DPFL algorithm. It integrates global update estimation into local updates to suppress client drift, aggregates block-wise means of second-moment statistics across clients to stabilize variance, and introduces a Bias-Corrected mechanism to eliminate second-moment estimation bias caused by DP. 
    \item \textbf{Theoretical and empirical guarantees.} We derive the first convergence guarantee for a DPFL adaptive optimizer that overcomes gradient heterogeneity, achieving a linear speedup rate of $\mathcal{O}(\sqrt{(L \Delta \sigma_l^2)/(S K T \epsilon^2)}+(L \Delta)/T+\sigma^2 G_{g}^2/s^2 R^2)$. Experiments on vision and language benchmarks demonstrate that DP-FedAdamW achieves a better privacy–utility trade-off and consistently outperforms SOTA DPFL baselines in convergence.
   
\end{itemize}

\section{Related work}
\label{sec:related}

\vspace{-2mm}
\begin{table*}[h]
	\centering
    \small
	\setlength{\tabcolsep}{1.3pt}
        \caption{Comparison with existing DPFL optimizers.}
	\vspace{-2mm}
	\renewcommand{\arraystretch}{0.8}
	\begin{tabular}{lcccc>{\columncolor{LightBlue}}c}
		\midrule[0.8pt]
		\multirow{1}{*}{\textbf{Method}} & \textbf{DP-FedAvg} & \textbf{DP-SCAFFOLD} & \textbf{DP-FedSAM}  &\textbf{DP-LocalAdamW}& \textbf{DP-FedAdamW} \\
		\midrule
		Optimizer Type   & SGD & SGD & SAM & AdamW & AdamW\\
        Scalability to Large Models   & Limited  & Limited & Moderate & Moderate & High \\
        Client Drift & Strong & Weakened & Slightly weakened & Slightly weakened & Weakened  \\
		Communication Cost & 1$\times$ & 2$\times$ & 2$\times$ & 1$\times$ & 1$\times$ \\
		\midrule[0.8pt]
	\end{tabular}
  \label{tab:related}
  \vspace{-2mm}
\end{table*}

\noindent
\textbf{Differentially Private Federated Learning (DPFL).}
We adopt sample-level DP for FL, ensuring privacy for each training example against any observer of the updates (the server or third parties).
The most prevalent DP-FedAvg~\cite{noble2022differentially} applies per-sample clipping and Gaussian noise locally before aggregation. DP-SCAFFOLD~\cite{noble2022differentially} introduces control variates to alleviate client drift under privacy constraints. DP-FedAvg-LS \cite{liang2024differentially} improves performance using Laplacian smoothing \cite{osher2022laplacian}. Beyond SGD, DP-FedSAM~\cite{shi2023make} uses Sharpness-Aware Minimization (SAM~\cite{foret2021sharpness}) to find flatter minima, improving robustness to DP noise. Despite progress in refined clipping and personalized DP~\cite{wei2025dc,liu2024cross}, optimization remains the dominant bottleneck. However, these algorithms all employ the SGD optimizer for local updates and do not consider the use of adaptive optimizers. In the optimization of large-scale models, the AdamW optimizer has been shown to outperform SGD significantly in both convergence speed and final performance \cite{zhao2026advances,zhao2024balf,zhao2023benchmark,liao2025convex,ding2026enhancing,ding2025traffic,an2024robust,an2024inspired,an2022numerical,liu2022kill,liu2023single, Feng2022SSR, Feng2023MaskCon, Feng2024NoiseBox, Feng2024CLIPCleaner,Feng2025OpenSet,Sun2024LAFS, Feng2026IdealNoise,Feng2025PROSAC,Feng2026NoisyValid,liu2026health,shen2025aienhanced,shen2026mftformer,sun2025objective,ShenSunQi2025,li2025frequency,li2025sepprune,li2026comprehensive,10605121,qi2026federated,qi2023cross,qi2025federated}.

\noindent
\textbf{Adaptive optimization.}
Adaptive optimizers such as Adam~\cite{kingma2014adam}, AdamW~\cite{loshchilov2017fixing}, AMSGrad~\cite{reddi2019convergence}, Amsgrad~\cite{reddi2018convergence} adapt per-parameter step sizes from gradient history, achieving faster and more stable convergence than vanilla SGD and becoming standard for large models (e.g., Transformers)~\cite{zhang2024transformers, NEURIPS2024_350e718f}.
Among these, AdamW applies weight decay directly on parameters, preventing undesired interaction between regularization and moment adaptation, thus improving convergence and generalization~\cite{loshchilov2017fixing, xie2024implicit,10.1145/3746027.3755226,liu2024fedbcgd,liuimproving}.
\noindent
\textbf{Adaptive optimization in Federated Learning.} 
A line of work brings adaptive updates into FL. FedOpt~\cite{reddiadaptive} introduces server-side adaptivity, instantiating Adam and Yogi at the aggregator. FAFED~\cite{wu2023faster} stabilizes training by aggregating clients’ first- and second-moment estimates of Adam. FedAMS~\cite{chen2020toward} highlights that averaging the second-moment is pivotal to avoid divergence. FedLADA~\cite{sun2023efficient} aggregates the second-moment estimate to reduce communication cost. However, these studies mainly evaluate on CNNs. However, these algorithms do not take into account the impact of noise introduced by DP. FedBCGD \cite{liu2024fedbcgd} proposes an accelerated block coordinate gradient descent framework for FL.
FedSWA \cite{liuimproving} improves generalization under highly heterogeneous data by stochastic weight averaging.
FedAdamW \cite{liu2025fedadamw} introduces a communication-efficient AdamW-style optimizer tailored for federated large models.
FedNSAM \cite{liu2025consistency} studies the consistency relationship between local and global flatness in FL.
FedMuon \cite{liu2025fedmuon} accelerates federated optimization via matrix orthogonalization.
DP-FedPGN \cite{liu2025dp} develops a penalizes gradient norms to encourage globally flatter minima in DP-FL.

\noindent\textbf{Our contributions.} 
To the best of our knowledge, adaptive optimization has not yet been systematically explored in the context of DPFL. Existing DPFL optimizers still struggle to train large models in a stable and efficient manner. To address this gap, we propose a method that stabilizes second-moment statistics via block-wise aggregation, removes DP-induced bias in the second moment, and aligns local updates with an estimated global descent direction in a communication-efficient way. We summarize representative DPFL methods and their characteristics in~\Cref{tab:related}.

\section{DPFL framework}
Consider a general DPFL system with $N$ clients. The FL objective is to minimize the following population risk:
\vspace{-2mm}
\begin{equation}
f(\boldsymbol{\theta}) = \frac{1}{N} \!\sum\nolimits_{i=1}^{N} f_i(\boldsymbol{\theta}),
f_i(\boldsymbol{\theta}) := \mathbb{E}_{\xi_i \sim \mathcal{P}_i}\big[ F_i(\boldsymbol{\theta};\xi_i) \big],
\end{equation}
where $f_i$ is the loss function of client $i$. $\xi_i$ is drawn from distribution $\mathcal{P}_i$ in client $i$, and $\mathbb{E}[\cdot]$ denotes expectation with respect to the sample $\xi_i$. 

In communication round $t$, the server selects $S$ clients and distributes the global model $\boldsymbol{\theta}^t$. The client $i$ computes the gradient for each sample $j$ in local $k$ iterations (with a mini-batch $\mathcal{D}_i^k{=}\lfloor s R\rfloor$, $s$ is data sampling rate, $R$ is number of client samples): $g_{ij}=\nabla f_i(\boldsymbol{\theta}_{i}^{t,k}; \xi_{ij})$, and performs gradient clipping and Gaussian noise addition to satisfy DP:
\vspace{-2mm}
\begin{equation}
    \bar{g}_{ij} = g_{ij} / \max(1, \|g_{ij}\|_2 / C),\label{bar_g}
\end{equation}
\begin{equation}
    \tilde g_i^{t,k} \leftarrow \frac{1}{s R} \sum_{j \in \mathcal{D}_i^k} \bar{g}_{i j}+\frac{ C}{s R}\mathcal{N}\!\left(0, \sigma^2 C^2 I\right),\label{tilde_g}
\end{equation}
where $C$ is the clipping threshold, $\sigma$ is the noise factor.
After $K$ local steps, the server updates the global model $\boldsymbol{\theta}^{t+1}$ by adding the averaged aggregate local update to $\boldsymbol{\theta}^t$.

\begin{figure}[h]
    \centering
    \subcaptionbox{Variance evolution of $v$~\label{fig:sub1}}{\includegraphics[width=0.236\textwidth]{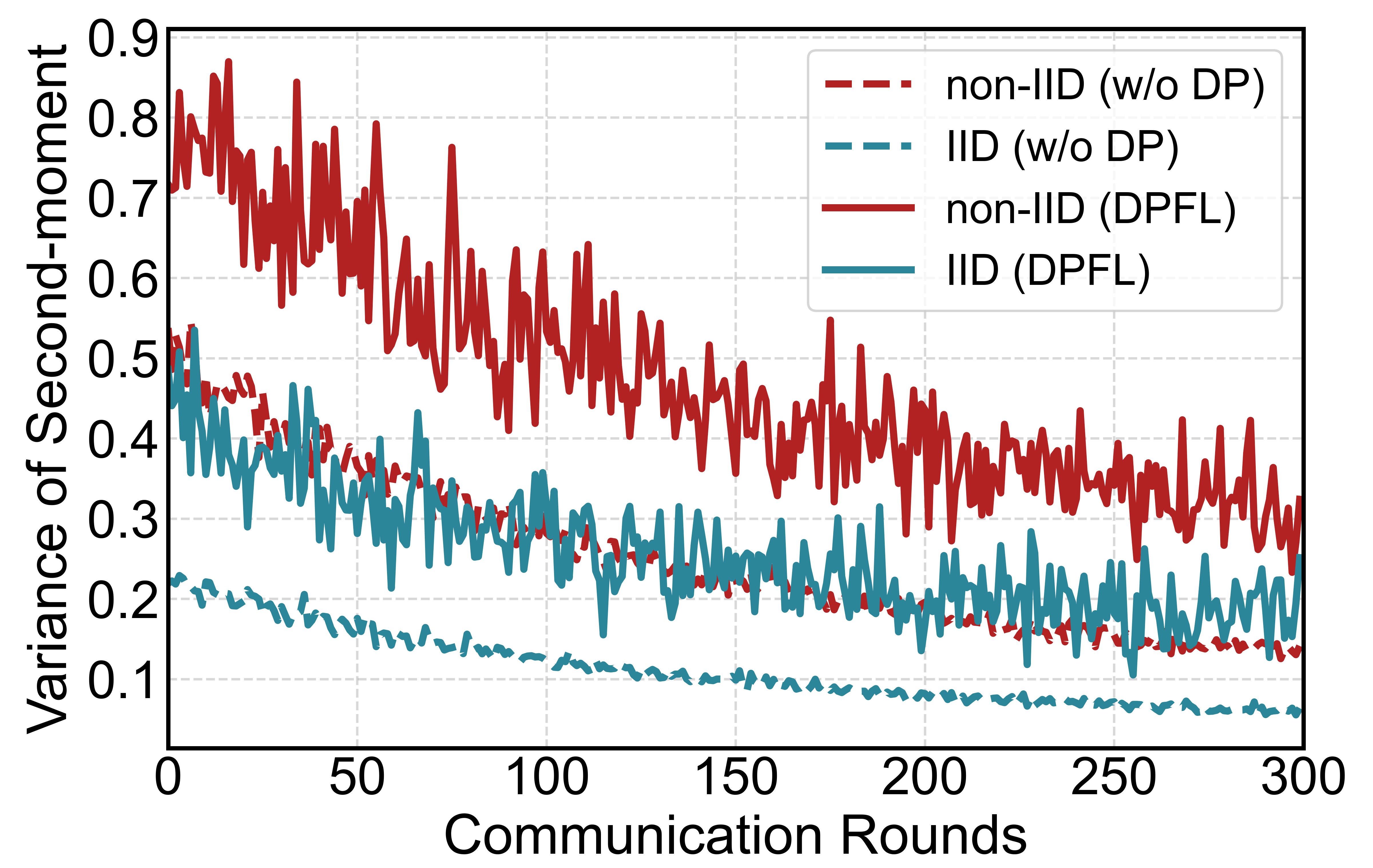}}
    \subcaptionbox{Client drift comparison~\label{fig:sub2}}{\includegraphics[width=0.236\textwidth]{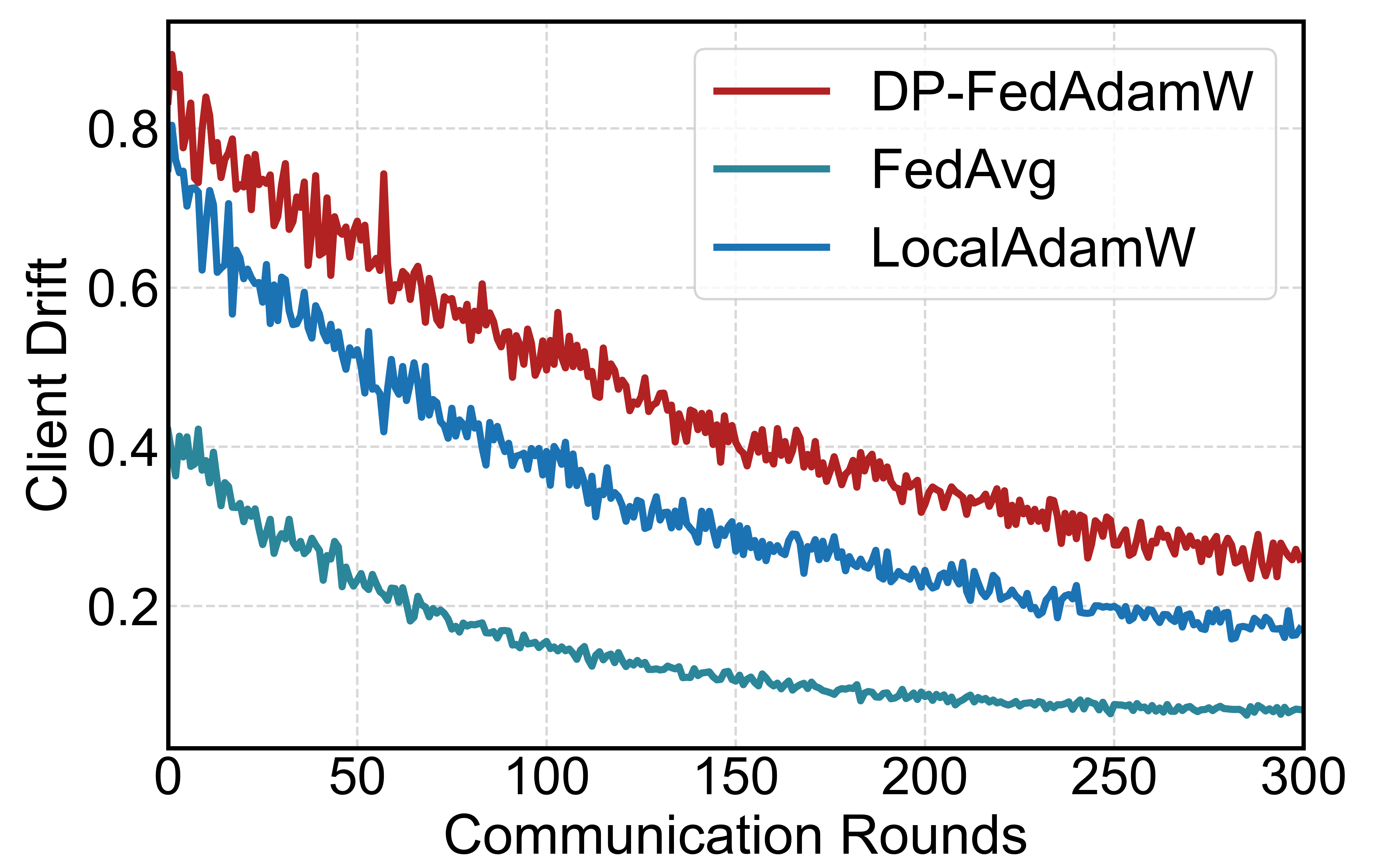}}
    \caption{Training on CIFAR-100, Swin-Tiny, $\sigma{=}1$, $\alpha{=}0.1$. (a) Non-IID (DPFL) causes high variance in second-moment estimator across clients of DP-LocalAdamW. (b) DP-LocalAdamW suffers from more severe client drift than FedAvg and LocalAdamW.}
    \label{fig:challenge1_3}
    \vspace{-2mm}
\end{figure}

\vspace{-4mm}
\section{Motivation and challenges}
\label{sec:challenges}
Prior work~\cite{loshchilov2017fixing, zhou2024towards} shows that AdamW can substantially improve convergence and generalization. However, our key observation is that, in DPFL, AdamW fails to exhibit its typical advantages. In fact, in transformer-based models, AdamW sometimes performs worse than SGD, indicating the benefits of AdamW are diminished when privacy constraints and the FL setup introduce additional complexities.

\noindent
\textbf{Challenge 1: Variance amplification of the second-moment estimator under non-IID and DP}.
In FL, non-IID data disperses client gradients, while DP noise further increases the variance of noisy gradients. This high variance accumulates in AdamW’s exponentially weighted moving averages of the first-moment $\boldsymbol{m}$ and second-moment $\boldsymbol{v}$:
\vspace{-1mm}
\begin{equation}
\begin{aligned}
\boldsymbol{m}_i^{t,k} &\leftarrow \beta_1 \boldsymbol{m}_i^{t,k-1} + (1-\beta_1)\,\tilde g_i^{t,k},\\
\boldsymbol{v}_i^{t,k} &\leftarrow \beta_2 \boldsymbol{v}_i^{t,k-1} + (1-\beta_2)\,\tilde g_i^{t,k} \odot \tilde g_i^{t,k},
\end{aligned}\label{m_v}
\end{equation}
where $\beta_1{=}0.9$, $\beta_2{=}0.999$ are the exponential decay rates. $\odot$ represents the elementwise (Hadamard) product. 
\vspace{-2mm}
\begin{align}
\operatorname{Var}(v^{t})
&\approx
\frac{(1-\beta_{2})^{2}}{1-\beta_{2}^{2}}\,
\operatorname{Var}\!\left(\tilde g^{t}\odot \tilde g^{t}\right).
\end{align}
Since $\tilde g^{t} \odot \tilde g^{t}$ magnifies noise and $\beta_2\!\approx\!1$ induces slow averaging, 
the variance of $\boldsymbol{v}$ grows rapidly and dominates the optimization noise under DP and heterogeneity. ~\Cref{fig:challenge1_3}(a) shows the variance of $\boldsymbol{v}$ under non-IID+DP remains higher and converges much slower than in IID or non-DP cases.

\begin{figure}[h]
\vspace{-1mm}
    \centering
    \subcaptionbox{Histogram of $\boldsymbol{m}^t$~\label{fig:sub1}}{\includegraphics[width=0.236\textwidth]{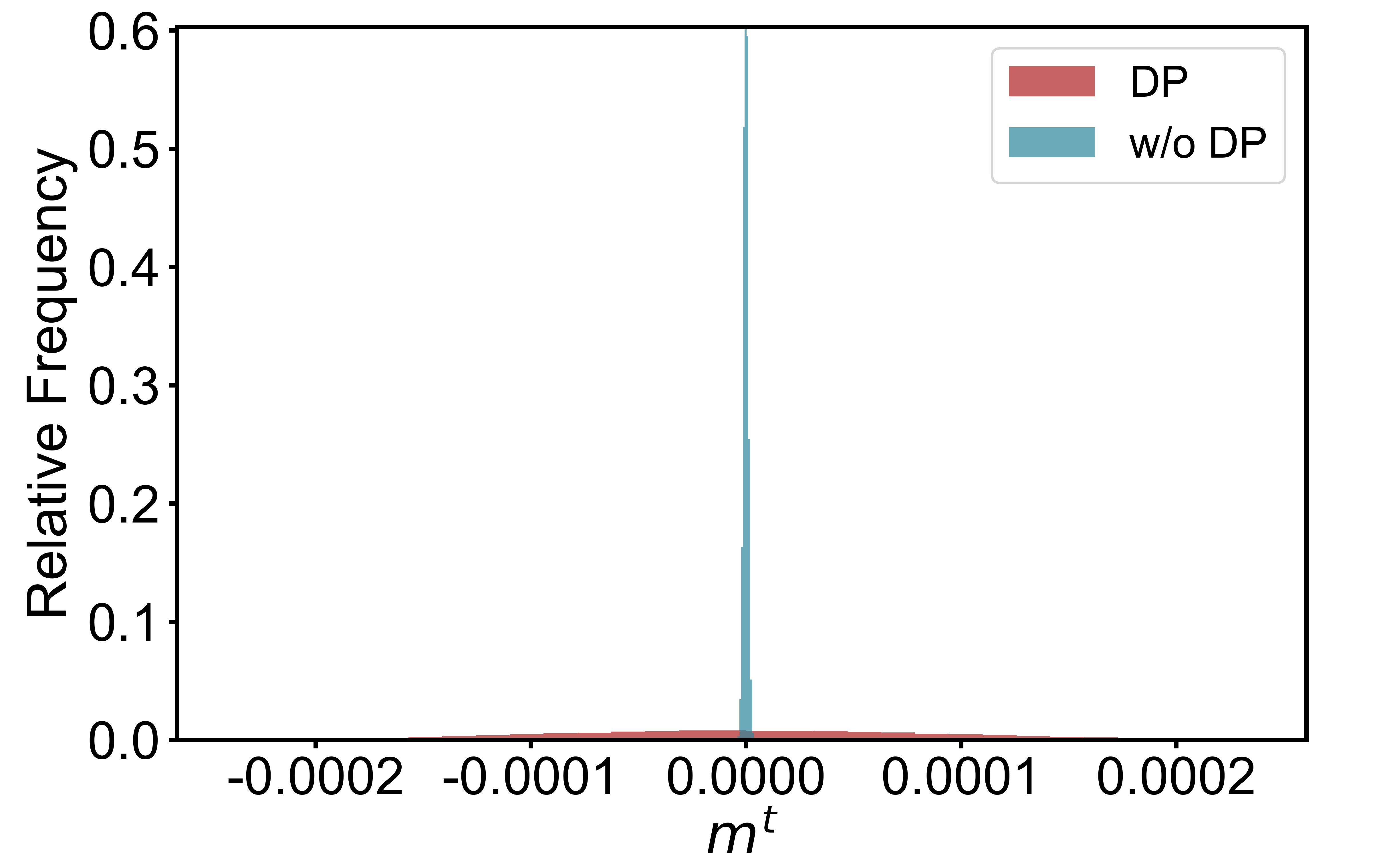}}
    \subcaptionbox{Histogram of $\sqrt{\boldsymbol{v}^t}$~\label{fig:sub2}}{\includegraphics[width=0.236\textwidth]{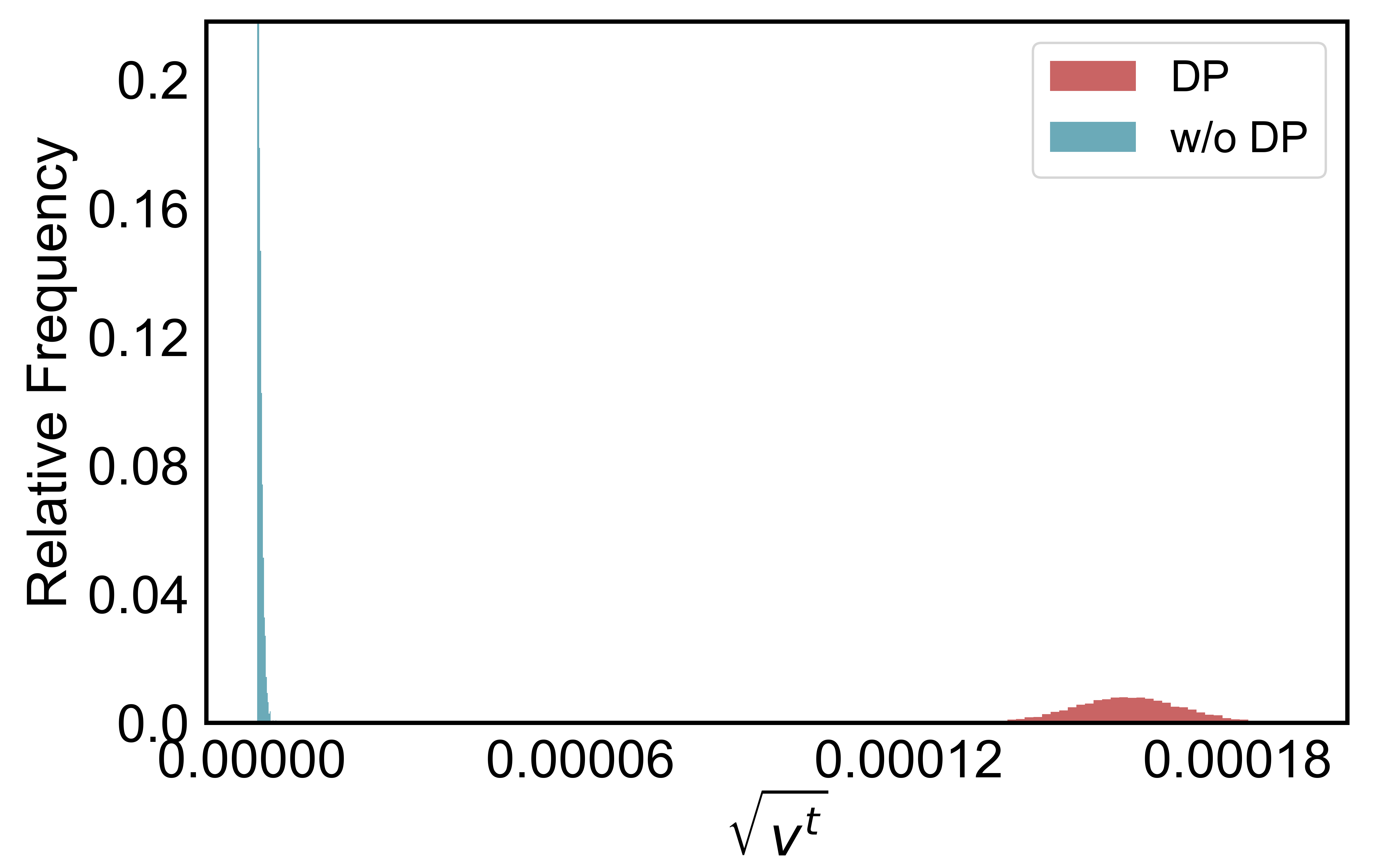}}
    \caption{Histogram for DP-LocalAdamW, CIFAR-10, Swin-Tiny, $\sigma{=}1$, $\alpha{=}0.1$. (a) The distribution centers of $\boldsymbol{m}^t$ are aligned with or without DP, but the variance is larger with DP. (b) The distribution of $\sqrt{\boldsymbol{v}^t}$ shows a significant difference, with the center of the distribution shifting approximately by $\sqrt{\sigma^2 C^2/(sR)^2}$.}
    \label{fig:challenge2}
    \vspace{-3mm}
\end{figure}

\noindent
\textbf{Challenge 2: Systematic bias in the second-moment estimator introduced by DP.} 
Under DP clipping and noise, directly using the perturbed gradient to estimate the second moment produces a systematic bias. This bias is additive and distinct from the initialization bias addressed by Adam’s exponential moving average, so that correction does not remove it in~\Cref{fig:challenge2}. Recall the definition of $\boldsymbol{v}_i^{t,k}$ from Eq.~\ref{m_v}, Eq.~\ref{bar_g}, Eq.~\ref{tilde_g}, we have 
\vspace{-2mm}
\begin{equation}
\begin{aligned}
   & \mathbb{E}\!\left[\tilde g_i^{t,k}\odot \tilde g_i^{t,k}\right]
    = \mathbb{E}\!\left[\bar g_i^{t,k}\odot \bar g_i^{t,k}\right]
       + {\sigma^2 C^2}/{(sR)^2}\, I,\\[3pt]
    &\mathbb{E}\!\left[v_i^{t,k}\right]
    = \mathbb{E}\!\left[v_i^{t,k}\right]_{\mathrm{w/o\ DP}}
       + \bigl(1-\beta_2^{\,k}\bigr)\,{\sigma^2 C^2}/{(sR)^2}\, I.
\end{aligned}
\end{equation}

\noindent
\textbf{Challenge 3: Client drift caused by local adaptivity and DP mechanism.} In FL, each client optimizes its local objective $f(\cdot)$, creating a gap between local and global optima. AdamW accelerates this by pushing clients toward their local optima, increasing client drift. In DPFL, gradient clipping and noise injection reduce effective sample sizes, making updates noisier and amplifying overfitting. This exacerbates client drift, hindering global model convergence. ~\Cref{fig:challenge1_3}(b) shows that DP-FedAdamW results in higher client drift compared to FedAvg and LocalAdamW, where LocalAdamW refers to Federated AdamW without DP.



\vspace{-2mm}
\begin{algorithm}[t]
\caption{DP-FedAdamW}
\label{alg:DP-FedAdamW}
\begin{algorithmic}[1]
\STATE \textbf{Initial} initial model $\boldsymbol{\theta}^0$; rounds $T$; local steps $K$; step size $\eta$; weight decay $\lambda$; AdamW $(\beta_1,\beta_2,\epsilon)$; clip norm $C$; noise multiplier $\sigma$; $S$ clients per round

\FOR{$t=1$ \TO $T$}
  \FOR{selected clients $i=1,\dots,S$ in parallel}
    \STATE $\boldsymbol{\theta}_i^{t,0} \leftarrow \boldsymbol{\theta}^t$;\ \ $\boldsymbol{m}_i^{t,0} \leftarrow \boldsymbol{0}$;\ \ \fcolorbox{LightBlue}{LightBlue}{$\boldsymbol{v}_i^{t,0} \leftarrow \boldsymbol{\bar{v}}^{t}$}
    \FOR{$k=1$ \TO $K$}
      \STATE sample $\mathcal{D}_i^k \subset\mathcal{D}_i$ of size $\lfloor s R\rfloor$
      \FOR{each sample $j \in \mathcal{D}_i^k$}
        \STATE $g_{ij} \leftarrow \nabla f_i(\boldsymbol{\theta}_i^{t,k};\xi_{ij})$
        \STATE $\bar{g}_{ij} \leftarrow g_{ij} / \max(1, \|g_{ij}\|_2 / C)$
      \ENDFOR
      \STATE $\tilde g_i^{t,k} \leftarrow \frac{1}{s R} \sum_{j \in \mathcal{D}_i^k} \bar{g}_{i j}+\frac{ C}{s R}\mathcal{N}\!\left(0, \sigma^2 C^2 I\right)$ 
      \STATE $\boldsymbol{m}_i^{t,k} \leftarrow \beta_1 \boldsymbol{m}_i^{t,k-1} + (1-\beta_1)\,\tilde g_i^{t,k}$
      \STATE $\boldsymbol{v}_i^{t,k} \leftarrow \beta_2 \boldsymbol{v}_i^{t,k-1} + (1-\beta_2)\,\tilde g_i^{t,k} \odot \tilde g_i^{t,k}$
      \STATE $\hat{\boldsymbol{m}}_i^{t,k} \leftarrow \boldsymbol{m}_i^{t,k}/(1-\beta_1^{k})$
      \STATE $\hat{\boldsymbol{v}}_i^{t,k} \leftarrow \boldsymbol{v}_i^{t,k}/(1-\beta_2^{k})$
      \STATE \fcolorbox{LightBlue}{LightBlue}{$\vartheta_{i}^{t,k} \leftarrow 1 / ( \sqrt{\hat{\boldsymbol{v}}_{i}^{t,k}- \left( \frac{\sigma C}{sR} \right)^2}+\epsilon)$}
      \STATE \fcolorbox{LightBlue}{LightBlue}{$\boldsymbol{\theta}_i^{t,k+1}\!\leftarrow\!\boldsymbol{\theta}_i^{t,k}\!-\! \eta(\hat{\boldsymbol{m}}_i^{t,k}\!\odot\!\vartheta_i^{t,k}\!+\!\gamma\boldsymbol{\Delta}_G^t\!-\!\lambda\boldsymbol{\theta}_i^{t,k})$}
    \ENDFOR

    \STATE \fcolorbox{LightBlue}{LightBlue}{Send $( \boldsymbol{\theta}^{t, K}_i\!\!-\!\!\boldsymbol{\theta}^{t, 0}_i, \boldsymbol{\bar{v}}_i\!=\!\texttt{Block\_mean}(\boldsymbol{v}^{t, K}_i) )$}
    \ENDFOR
    \STATE $\boldsymbol{\Delta}_G^{t+1}=\frac{-1}{SK\eta} \sum_{i=1}^S (\boldsymbol{\theta}^{t, K}_i-\boldsymbol{\theta}^{t, 0}_i)$
    \STATE $\boldsymbol{\theta}^{t+1} =\boldsymbol{\theta}^{t} +\frac{1}{S} \sum_{i=1}^S (\boldsymbol{\theta}^{t, K}_i-\boldsymbol{\theta}^{t, 0}_i)$
    \STATE \fcolorbox{LightBlue}{LightBlue}{$\boldsymbol{\bar{v}}^{t+1} =\frac{1}{S} \sum_{i=1}^S \boldsymbol{\bar{v}}_i$}
    \STATE Broadcast $(\boldsymbol{\theta}^{t+1}, \boldsymbol{\bar{v}}^{t+1},\boldsymbol{\Delta}_G^{t+1} ) $ to clients
\ENDFOR

\end{algorithmic}
\end{algorithm}

\section{Methodology}
To restore the effectiveness of AdamW in DPFL, we propose Differentially Private Federated AdamW (DP-FedAdamW). DP-FedAdamW is built upon a unified perspective: 
\emph{the moment statistics of AdamW become distorted under non-IID data and DP.} Our design corrects this distortion along three complementary axes: parameter-block aggregation to stabilize second-moment variance, 
explicit removal of DP-induced bias in the second moment, 
and local–global update alignment to reduce client drift.

\subsection{Second-moment aggregation}
To address \textbf{Challenge 1}, we stabilize the training process by aggregating the second-moment estimates of AdamW. However, directly aggregating per-parameter second moments requires communicating an additional vector with the same dimensionality as the model parameters, which almost doubles the communication overhead. We therefore introduce a block-wise aggregation scheme that communicates only \emph{one} statistic per parameter block.

We adopt a parameter-block averaging strategy, where each block shares a single learning rate to capture local curvature effectively. Based on this idea, we partition the parameters into $B$ blocks and compute the \texttt{mean} of second-moments within each block (line 19 in Algorithm~\ref{alg:DP-FedAdamW}):
\vspace{-2mm}
\begin{equation}
\boldsymbol{\bar{v}}_b = \frac{1}{|\boldsymbol{v}_b|} \sum_{i \in \boldsymbol{v}_b} v_i, \quad b = 1, \dots, B.
\end{equation}

\noindent
\textbf{Block-wise averaging strategy.}
We design the blocks to align with the model architecture. The attention-related parameters are partitioned at the granularity of attention heads, while the remaining modules are partitioned at the granularity of whole layers. For each parameter block, we compute the \texttt{mean} 
of its second-moment estimates and transmit only this block-wise statistic:
\begin{itemize}
    \item \textbf{Head-wise \texttt{Q}, \texttt{K}, \texttt{V}.} 
    According to the number of attention heads in the model, each head-specific slice of the
    \texttt{query} (\texttt{Q}), \texttt{key} (\texttt{K}), and \texttt{value} (\texttt{V}) 
    weight matrix is treated as an independent parameter block. For every such 
    Q/K/V block, we compute and transmit the \texttt{mean} of its second-moment estimates.
    \item \textbf{\texttt{attn.proj} and MLP layers.} 
    Each attention output projection layer (\texttt{attn.proj}) and each MLP layer is regarded as a single parameter block. For each layer, we aggregate the second-moment estimates of all parameters within the layer and transmit the block-wise \texttt{mean}.
    \item \textbf{\texttt{Embedding} and \texttt{output} layers.} 
    Similarly, each \texttt{Embedding} layer and each \texttt{output} (e.g., classification) layer is treated as one parameter block. For these blocks, we also compute and transmit the \texttt{mean} of the corresponding second-moment estimates.
\end{itemize}

\textbf{CNNs (e.g., ResNet):} For convolutional networks, we define each convolutional layer or residual block as one parameter block. This reduces the communication cost from per-parameter second moments to $B$ block-level values while preserving adaptive behavior across different components of the network. 

\subsection{Unbiased second-moment correction}
To address \textbf{Challenge 2}, we propose a Bias-Corrected (BC). The correction term subtracts the variance contribution of the noise from the second-moment estimate of DP-AdamW, restoring the scaling behavior of non-private Adam and removing the influence of DP on the second-moment estimation. The update rule is given by (line 16 in Algorithm \ref{alg:DP-FedAdamW}):
\vspace{-3mm}
\begin{equation}
\vartheta_{i}^{t,k}=1 / ( \sqrt{\hat{\boldsymbol{v}}_{i}^{t,k}- \left( \frac{\sigma C}{sR} \right)^2}+\epsilon)
\end{equation}
Here, $\hat{\boldsymbol{v}}_{i}^{t,k}$ denotes the exponential moving average of the (noisy) second moment of the gradient for parameter $i$, $\sigma$ is the noise multiplier, $C$ is the clipping norm, and $sR$ is the batch size. The term $\frac{\sigma C}{sR}$corresponds to the variance of the Gaussian noise added for differential privacy, and subtracting it removes the constant bias introduced by DP noise in the second-moment estimate. 

\subsection{Local-global alignment}

To mitigate \textbf{Challenge~3} (client drift under non-IID data with DP), we augment the local AdamW update with an explicit alignment toward the global descent direction. Concretely, we modify the local update rule (line~17 in Algorithm~\ref{alg:DP-FedAdamW}) as
\vspace{-1mm}
\begin{equation}
\boldsymbol{\theta}_i^{t,k+1} = \boldsymbol{\theta}_i^{t,k} - \eta \left( \hat{\boldsymbol{m}}^{t,k}_i \odot \vartheta_i^{t,k} - \lambda \boldsymbol{\theta}_i^{t,k} + \gamma \boldsymbol{\Delta}_G^t \right),
\end{equation}
where $\boldsymbol{\Delta}_G^t=\frac{-1}{SK\eta} \sum_{i=1}^S (\boldsymbol{\theta}_i^{t,K} \!-\! \boldsymbol{\theta}_i^{t,0})$ is an empirical estimate of the global update over $S$ participating clients.
 
Under non-IID data, the adaptive local gradient $\hat{\boldsymbol{m}}^{t,k}_i \odot \vartheta_i^{t,k}$ is biased toward each client's own objective. DP clipping distorts the magnitude and direction of these gradients, and DP noise adds high-variance perturbations, pushing the local parameters $\boldsymbol{\theta}_i^{t,k}$ toward different client-specific optima and amplifying the client drift identified in \textbf{Challenge~3}.
The alignment term $\gamma \boldsymbol{\Delta}_G^t$ softly regularizes each local step toward the global descent direction. It is nearly inactive when a client's update already follows the global trend, but when non-IID data and DP operations drive the local direction away, the alignment term counteracts this deviation and steers the trajectory back toward the global path. As illustrated in Figure~\ref{figure method}, this local--global alignment tightens the spread of client models, reduces cross-client update variance at aggregation, and yields a smoother, more stable global optimization trajectory under strong non-IID and DP.

\section{Theoretical analysis}

\subsection{Convergence analysis}

\label{convergence_analysis}
In this part, we give the convergence theoretical analysis of our proposed DP-FedAdamW algorithm. Firstly, we give some standard assumptions for the non-convex function $f$.
\begin{assumption}[Smoothness]
	\label{smoothness}
	(Smoothness) \textit{The non-convex $f_{i}$ is a $L$-smooth function for all $i\in[m]$, i.e., $\Vert\nabla f_{i}(\boldsymbol{\theta}_1)-\nabla f_{i}(\boldsymbol{\theta}_2)\Vert\leq L\Vert\boldsymbol{\theta}_1-\boldsymbol{\theta}_2\Vert$, for all $\boldsymbol{\theta}_1,\boldsymbol{\theta}_2\in\mathbb{R}^{d}$.}
\end{assumption}
\begin{assumption}[Bounded Stochastic Gradient]
	\label{bounded_stochastic_gradient_I}
    \textit{$\boldsymbol{g}_{i}^{t}=\nabla f_{i}(\boldsymbol{\theta}_{i}^{t}, \xi_i^{t})$ computed by using a sampled mini-batch data $\xi_i^{t}$ in the local client $i$ is an unbiased estimator of $\nabla f_{i}$ with bounded variance, i.e., $\mathbb{E}_{\xi_i^{t}}[\boldsymbol{g}_{i}^{t}]=\nabla f_{i}(\boldsymbol{\theta}_{i}^{t})$ and     $\mathbb{E}_{\xi_i^{t}}\Vert g_{i}^{t} - \nabla f_{i}(\boldsymbol{\theta}_{i}^{t})\Vert^{2} \leq \sigma_{l}^{2}$, for all $\boldsymbol{\theta}_{i}^{t}\in\mathbb{R}^{d}$.}
\end{assumption}
\begin{assumption}[Bounded Stochastic Gradient II]
	\label{bounded_stochastic_gradient_II}
	 \textit{Each element of stochastic gradient $\boldsymbol{g}_{i}^{t}$ is bounded, i.e., $\Vert\boldsymbol{g}_{i}^{t}\Vert_{\infty}=\Vert f_{i}(\boldsymbol{\theta}_{i}^{t},\xi_i^{t})\Vert_{\infty}\leq G_{g}$, for all $\boldsymbol{\theta}_{i}^{t}\in\mathbb{R}^{d}$ and any sampled mini-batch data $\xi_i^{t}$.}
\end{assumption}

\begin{theorem}[Convergence for non-convex functions]\label{theorem_convergence_rate1}
	Under Assumptions \ref{smoothness}, \ref{bounded_stochastic_gradient_I}, and \ref{bounded_stochastic_gradient_II}, if we take $g^0=0$,
	then DP-FedAdamW converges as follows
     \vspace{-2mm}
	\begin{equation}
    \frac{1}{T}\!\sum_{t=0}^{T-1}\!
    \mathbb{E}\!\left[\|\nabla f(\boldsymbol{\theta}^{t})\|^2\right]
    \lesssim
    \mathcal{O}\!\left(
    \sqrt{\frac{L\Delta\sigma_l^2}{SKT\epsilon^2}}
    +\frac{L\Delta}{T}
    +\frac{\sigma^2 G_{g}^2}{s^2 R^2}
    \right).
\end{equation}
	Here $G_0:=\frac{1}{N} \sum_{i=1}^N\left\|\nabla f_i\left(\boldsymbol{\theta}^0\right)\right\|^2$,$\Delta=f\left(\boldsymbol{\theta}^0\right)-f^{\star} $, $S$ is the number of participating clients per round, $\sigma$ is DP noise level, $\sigma_l$ is noise of stochastic gradient, $K$ is the number of local iterations, and $T$ is the total number of communication rounds. 
\end{theorem}	
  
The proof is provided in Appendix~\ref{app:convergence_analysis}. The convergence rate of DP-FedAdamW is faster than that of DP-LocalAdamW's $\mathcal{O}\left(\sqrt{\frac{L \Delta (\sigma_l^2+\sigma_g^2)}{S K T \epsilon^2}}+\frac{L \Delta}{T}+\frac{\sigma^2 G_{g}^2}{s^2 R^2} \right)$, and we do not need bounded heterogeneity assumption. This is due to the suppression of local drift by the global update estimation $\boldsymbol{\Delta}_G^t$. Unlike prior adaptive FL optimizers \cite{sun2023efficient} which rely on heterogeneity assumption, our convergence result holds without such restrictions, making it more widely applicable.

\subsection{Privacy analysis}
Differential privacy (DP) \cite{dwork2014algorithmic} provides a formal framework for quantifying disclosure risk. We conduct privacy accounting under R\'enyi DP (RDP)~\cite{mironov2017renyi}, which gives tight bounds under composition and subsampling. Following~\cite{noble2022differentially}, we finally convert RDP to $(\varepsilon,\delta)$-DP~\cite{mironov2017renyi}.

\begin{definition}(Sample-level DP, \cite{fu2024differentially}). 
    For any two neighboring datasets $\mathcal{D}$ and $\mathcal{D}^{\prime}$ that differ in a single sample or record (either by addition or removal), a randomized mechanism $\mathcal M:\mathcal X^n\!\to\!\mathcal Y$ satisfies $(\varepsilon, \delta)$-DP at the sample level if, for every possible output set $O \subseteq \mathbb{Y}$, the following holds:
    \vspace{-1mm}
	\begin{equation}
		\operatorname{Pr}[\mathcal{M}(\mathcal{D}) \in O] \leq e^{\varepsilon} \operatorname{Pr}\left[\mathcal{M}\left(\mathcal{D}^{\prime}\right) \in O\right]+\delta,
	\end{equation}
where $\varepsilon$ is the privacy budget, $\delta$ is the failure probability.
\end{definition}
\vspace{-1mm}
\begin{definition}
(Rényi DP, \cite{mironov2017renyi}). For any order $\zeta \in (1, \infty )$, $\mathcal{M}$ satisfies $(\zeta, \varepsilon_{\mathrm T})$-RDP if for any two neighboring datasets $\mathcal{D}$, $\mathcal{D}^\prime$ that differ by a single sample, it has:
\vspace{-2mm}
\begin{equation}
\!\!\!D_{\zeta}\left[\mathcal{M}(\mathcal{D}) \| \mathcal{M}\left(\mathcal{D}^{\prime}\right)\right]\! :=\!\frac{1}{\zeta\!-\!1} \log \mathbb{E}\left[\left(\frac{p_{\mathcal{M}(\mathcal{D})}\left( Y\right)}{p_{\mathcal{M}\left(\mathcal{D}^{\prime}\right)}\left( Y\right)}\right)^{\zeta}\!\right] \! \leq \!\varepsilon_{\mathrm R},
\end{equation}
The expectation $\mathbb{E}$ is taken over the output $Y \sim \mathcal{M}(\mathcal{D}^\prime)$.
\end{definition}

We adopt standard RDP accounting following~\cite{noble2022differentially}, the cumulative privacy over $K$ local steps and $T$ rounds is:
\begin{theorem} [Privacy guarantee] For DP-FedAdamW, $\boldsymbol\theta^T$ is  $(\varepsilon, \delta)$-DP towards a third party
\vspace{-2mm}
\begin{equation} \label{eq:13}
\varepsilon=\mathcal{O}\left(\frac{s \sqrt{T K \log (2 / \delta) \log (2 T / \delta)}}{\sigma}\right),
\vspace{-1mm}
\end{equation}
Towards the server, accumulative $(\varepsilon_s, \delta_s)$-DP is
\vspace{-2mm}
\begin{equation} \label{eq:14}
\varepsilon_s=\varepsilon \sqrt{\frac{N}{l}}, \delta_s=\frac{\delta}{2}\left(\frac{1}{l}+1\right),
\end{equation}
where $N$ is the number of clients, $l$ is clients sample rate, and $s$ is client data sample rate.
\end{theorem}  

\begin{table*}[tb]
	\centering
    \small
	\setlength{\tabcolsep}{1.3pt}
    \caption{Averaged test accuracy (\%) and privacy budget comparison on CIFAR-10 and CIFAR-100 using ResNet-18 and ViT-Base. Assuming DP-FedAvg is $(\varepsilon,\delta)$-DP, all other methods share this budget except DP-FedSAM, which is approximately $(2\varepsilon,\delta)$-DP.}
	\vspace{-2mm}
	\renewcommand{\arraystretch}{0.8}
	\begin{tabular}{lccccccccc}
		\midrule[0.8pt]
		\multirow{2}{*}{\textbf{Method}} 
        & \multicolumn{2}{c}{\textbf{CIFAR-10 (ResNet-18)}} 
        & \multicolumn{2}{c}{\textbf{CIFAR-100 (ResNet-18)}} 
        & \multicolumn{2}{c}{\textbf{CIFAR-10 (ViT-Base)}} 
        & \multicolumn{2}{c}{\textbf{CIFAR-100 (ViT-Base)}} 
        & \multirow{2}{*}{\shortstack{\textbf{Privacy} \\ \textbf{Budget}}}\\
		\cmidrule(lr){2-3} \cmidrule(lr){4-5} \cmidrule(lr){6-7} \cmidrule(lr){8-9}
		   & $\alpha{=}0.6$ & $\alpha{=}0.1$   
           & $\alpha{=}0.6$ & $\alpha{=}0.1$
           & $\alpha{=}0.6$ & $\alpha{=}0.1$   
           & $\alpha{=}0.6$ & $\alpha{=}0.1$ \\
		\midrule
		DP-FedAvg      
            & $71.69_{\pm0.37}$ & $60.38_{\pm0.21}$ 
            & $31.13_{\pm0.57}$ & $25.04_{\pm0.72}$  
            & $90.51_{\pm0.37}$ & $78.85_{\pm0.21}$ 
            & $45.86_{\pm0.57}$ & $13.41_{\pm0.72}$   & $(\varepsilon,\delta)$\\
        DP-SCAFFOLD    
            & $71.79_{\pm0.44}$ & $59.04_{\pm0.39}$ 
            & $31.52_{\pm0.36}$ & $24.94_{\pm0.58}$  
            & $90.88_{\pm0.52}$ & $80.69_{\pm0.64}$ 
            & $62.45_{\pm0.57}$ & $64.69_{\pm0.76}$   & $(\varepsilon,\delta)$\\
        DP-FedAvg-LS   
            & $72.27_{\pm0.57}$ & $60.84_{\pm0.45}$ 
            & $35.91_{\pm0.54}$ & $28.68_{\pm0.52}$  
            & $91.48_{\pm0.47}$ & $88.48_{\pm0.50}$ 
            & $67.76_{\pm0.58}$ & $65.85_{\pm0.63}$    & $(\varepsilon,\delta)$\\
        DP-FedSAM      
            & $67.39_{\pm0.41}$ & $57.04_{\pm0.46}$ 
            & $30.36_{\pm0.50}$ & $24.03_{\pm0.66}$  
            & $91.76_{\pm0.55}$ & $89.35_{\pm0.65}$ 
            & $70.24_{\pm0.63}$ & $66.27_{\pm0.65}$   & $(2\varepsilon,\delta)$\\
		DP-LocalAdamW  
            & $71.73_{\pm0.47}$ & $60.55_{\pm0.53}$ 
            & $34.92_{\pm0.69}$ & $28.70_{\pm0.81}$  
            & $91.16_{\pm0.47}$ & $89.85_{\pm0.53}$ 
            & $68.21_{\pm0.69}$ & $65.24_{\pm0.81}$   & $(\varepsilon,\delta)$\\
        \rowcolor{gray!10}
		DP-FedAdamW    
            & $\color{red}{74.81}_{\pm0.19}$ & $\color{red}{64.65}_{\pm0.24}$ 
            & $\color{red}{39.56}_{\pm0.42}$ & $\color{red}{33.55}_{\pm0.50}$ 
            & $\color{red}{92.02}_{\pm0.49}$ & $\color{red}{90.30}_{\pm0.24}$ 
            & $\color{red}{71.59}_{\pm0.42}$ & $\color{red}{67.56}_{\pm0.50}$ & $(\varepsilon,\delta)$\\
		\midrule[0.8pt]
	\end{tabular}
  \label{tab:resnet18_vitbase}
  \vspace{-2mm}
\end{table*}

\begin{table*}[tb]
	\centering
    \small
	\setlength{\tabcolsep}{1.3pt}
    \caption{Averaged test accuracy (\%) and privacy budget comparison on CIFAR-100 and Tiny-ImageNet using Swin-Tiny and Swin-Base.}
	\vspace{-2mm}
	\renewcommand{\arraystretch}{0.9}
	\begin{tabular}{lccccccccc}
		\midrule[0.8pt]
		\multirow{2}{*}{\textbf{Method}} 
        & \multicolumn{2}{c}{\textbf{CIFAR-100 (Swin-T)}} 
        & \multicolumn{2}{c}{\textbf{CIFAR-100 (Swin-B)}} 
        & \multicolumn{2}{c}{\textbf{Tiny-ImageNet (Swin-T)}} 
        & \multicolumn{2}{c}{\textbf{Tiny-ImageNet (Swin-B)}} 
        & \multirow{2}{*}{\shortstack{\textbf{Privacy} \\ \textbf{Budget}}}\\
		\cmidrule(lr){2-3} \cmidrule(lr){4-5} \cmidrule(lr){6-7} \cmidrule(lr){8-9}
		   & $\alpha{=}0.6$ & $\alpha{=}0.1$   
           & $\alpha{=}0.6$ & $\alpha{=}0.1$
           & $\alpha{=}0.6$ & $\alpha{=}0.1$   
           & $\alpha{=}0.6$ & $\alpha{=}0.1$ \\
		\midrule
		DP-FedAvg     
            & $36.96_{\pm0.50}$ & $28.07_{\pm0.31}$ 
            & $65.61_{\pm0.27}$ & $60.77_{\pm0.32}$
            & $18.82_{\pm0.94}$ & $10.12_{\pm1.36}$ 
            & $52.04_{\pm0.66}$ & $36.19_{\pm0.74}$ & $(\varepsilon,\delta)$\\
        DP-SCAFFOLD      
            & $41.68_{\pm0.64}$ & $33.52_{\pm0.65}$ 
            & $68.03_{\pm0.55}$ & $65.34_{\pm0.49}$
            & $21.91_{\pm0.78}$ & $14.54_{\pm0.72}$ 
            & $53.78_{\pm0.79}$ & $43.55_{\pm0.60}$ & $(\varepsilon,\delta)$\\
        DP-FedAvg-LS    
            & $41.95_{\pm0.53}$ & $36.72_{\pm0.65}$ 
            & $68.80_{\pm0.52}$ & $66.08_{\pm0.56}$
            & $23.03_{\pm0.64}$ & $15.99_{\pm0.68}$ 
            & $55.83_{\pm0.66}$ & $47.46_{\pm0.57}$ & $(\varepsilon,\delta)$\\
        DP-FedSAM     
            & $44.19_{\pm0.48}$ & $38.43_{\pm0.62}$ 
            & $70.44_{\pm0.57}$ & $66.28_{\pm0.64}$
            & $22.75_{\pm0.65}$ & $16.86_{\pm0.77}$ 
            & $56.12_{\pm0.74}$ & $48.08_{\pm0.62}$ & $(2\varepsilon,\delta)$\\
		DP-LocalAdamW  
            & $42.04_{\pm0.63}$ & $35.51_{\pm0.58}$ 
            & $68.33_{\pm0.33}$ & $65.28_{\pm0.39}$
            & $23.41_{\pm0.87}$ & $15.81_{\pm0.73}$ 
            & $54.58_{\pm0.60}$ & $46.52_{\pm0.64}$ & $(\varepsilon,\delta)$\\
        \rowcolor{gray!10}
		DP-FedAdamW  
            & $\color{red}{46.77}_{\pm0.38}$ & $\color{red}{39.36}_{\pm0.55}$ 
            & $\color{red}{71.58}_{\pm0.30}$ & $\color{red}{67.45}_{\pm0.26}$
            & $\color{red}{23.96}_{\pm0.50}$ & $\color{red}{19.19}_{\pm0.55}$ 
            & $\color{red}{58.44}_{\pm0.54}$ & $\color{red}{50.76}_{\pm0.51}$ & $(\varepsilon,\delta)$\\
		\midrule[0.8pt]
	\end{tabular}
  \label{tab:swin_cifar100_tiny}
  \vspace{-2mm}
\end{table*}

\vspace{-2mm}
\begin{table}[tb]
	\centering
    \small
	\setlength{\tabcolsep}{1.3pt}
        \caption{Averaged test accuracy (\%) on RoBERTa-Base.}
	\vspace{-2mm}
	\renewcommand{\arraystretch}{0.8}
	\begin{tabular}{lcccc}
		\midrule[0.8pt]
		\multirow{1}{*}{\textbf{Method}} & \textbf{SST-2} & \textbf{QQP}  &\textbf{QNLI}& \textbf{MNLI} \\
		\midrule
		DP-FedAvg    & $49.08_{\pm0.12}$ & $63.18_{\pm0.50}$ & $49.46_{\pm0.44}$ & $35.45_{\pm0.20}$  \\
        DP-SCAFFOLD    & $52.35_{\pm0.18}$ & $66.87_{\pm0.42}$ & $52.62_{\pm0.52}$ & $38.74_{\pm0.32}$  \\
        DP-FedAvg-LS    & $68.38_{\pm0.50}$ & $72.59_{\pm0.48}$ & $69.22_{\pm0.68}$ & $50.46_{\pm0.55}$  \\
        DP-FedSAM    & $70.95_{\pm0.62}$ & $75.66_{\pm0.53}$ & $70.54_{\pm0.63}$ & $56.88_{\pm0.56}$  \\
		DP-LocalAdamW  & $ 91.04_{\pm0.53}$ & $82.33_{\pm0.59}$ & $86.07_{\pm0.47}$ & $75.20_{\pm0.44}$  \\
        \rowcolor{gray!10}
		DP-FedAdamW & $\color{red}{91.17}_{\pm0.32}$ & $\color{red}{83.34}_{\pm0.39}$ & $\color{red}{  86.33}_{\pm0.34}$ & $\color{red}{ 78.68}_{\pm0.38}$ \\
		\midrule[0.8pt]
	\end{tabular}
  \label{tab:roberta_base}
  \vspace{-2mm}
\end{table}

\begin{table*}[tb]
	\centering
    \small
	\setlength{\tabcolsep}{1.3pt}
        \caption{Effect of privacy budgets ($\varepsilon$) on test accuracy (\%) using Swin-Base, $\alpha=0.1$.}
	\vspace{-2mm}
	\renewcommand{\arraystretch}{0.95}
	\begin{tabular}{lccccccccc}
		\midrule[0.8pt]
		\multirow{2}{*}{\textbf{Method}} & \multicolumn{3}{c}{\textbf{CIFAR-10}} & \multicolumn{3}{c}{\textbf{CIFAR-100}} & \multicolumn{3}{c}{\textbf{Tiny-ImageNet}}\\
		\cmidrule(lr){2-4} \cmidrule(lr){5-7} \cmidrule(lr){8-10}
		   & $\varepsilon{=}1$ & $\varepsilon{=}2$   &$\varepsilon{=}3$ & $\varepsilon{=}1$ & $\varepsilon{=}2$ & $\varepsilon{=}3$   &$\varepsilon{=}1$ & $\varepsilon{=}2$ & $\varepsilon{=}3$\\
		\midrule
		DP-FedAvg    & $66.52_{\pm0.89}$ & $72.63_{\pm0.55}$ & $88.02_{\pm0.76}$ & $19.66_{\pm1.52}$  & $42.98_{\pm0.60}$ & $57.73_{\pm0.55}$ & $20.65_{\pm0.77}$ & $37.48_{\pm0.54}$ & $45.32_{\pm0.47}$\\
        DP-SCAFFOLD    & $66.39_{\pm0.68}$ & $72.81_{\pm0.53}$ & $88.07_{\pm0.66}$ & $22.38_{\pm0.67}$  & $44.52_{\pm0.68}$ & $59.85_{\pm0.52}$ & $20.93_{\pm0.66}$ & $38.40_{\pm0.51}$ & $47.06_{\pm0.30}$\\
        DP-FedAvg-LS    & $71.57_{\pm1.02}$ & $75.02_{\pm0.59}$ & $88.96_{\pm0.54}$ & $25.77_{\pm0.61}$  & $49.54_{\pm0.59}$ & $62.46_{\pm0.53}$ & $26.89_{\pm0.46}$ & $46.57_{\pm0.44}$ & $48.83_{\pm0.21}$\\
        DP-FedSAM    & $53.51_{\pm1.39}$ & $59.33_{\pm0.86}$ & $76.82_{\pm0.97}$ & $12.33_{\pm0.64}$  & $28.32_{\pm0.49}$ & $46.93_{\pm0.50}$ & $10.88_{\pm0.95}$ & $26.76_{\pm0.78}$ & $41.69_{\pm0.45}$\\
		DP-LocalAdamW & $66.57_{\pm1.35}$ & $73.04_{\pm0.58}$ & $88.83_{\pm0.71}$ & $24.69_{\pm0.88}$   & $47.26_{\pm0.59}$ & $62.18_{\pm0.61}$ & $28.40_{\pm0.72}$ & $46.62_{\pm0.37}$ & $48.55_{\pm0.28}$\\
		\rowcolor{gray!10} 
        DP-FedAdamW & $\color{red}{77.50}_{\pm0.54}$ & $\color{red}{80.37}_{\pm0.58}$ & $\color{red}{89.98}_{\pm0.55}$ & $\color{red}{30.06}_{\pm0.41}$ &$\color{red}{52.11}_{\pm0.54}$ & $\color{red}{65.47}_{\pm0.42}$ & $\color{red}{34.23}_{\pm0.50}$ & $\color{red}{50.95}_{\pm0.26}$ & $\color{red}{52.53}_{\pm0.22}$\\
		\midrule[0.8pt]
	\end{tabular}
  \label{tab:swin_base_epsilon}
  \vspace{-2mm}
\end{table*}
\begin{figure*}[tb]
    \centering
    \subcaptionbox{ResNet-18, $\alpha= 0.6$ \label{fig:sub1}}{\includegraphics[width=0.24\textwidth]{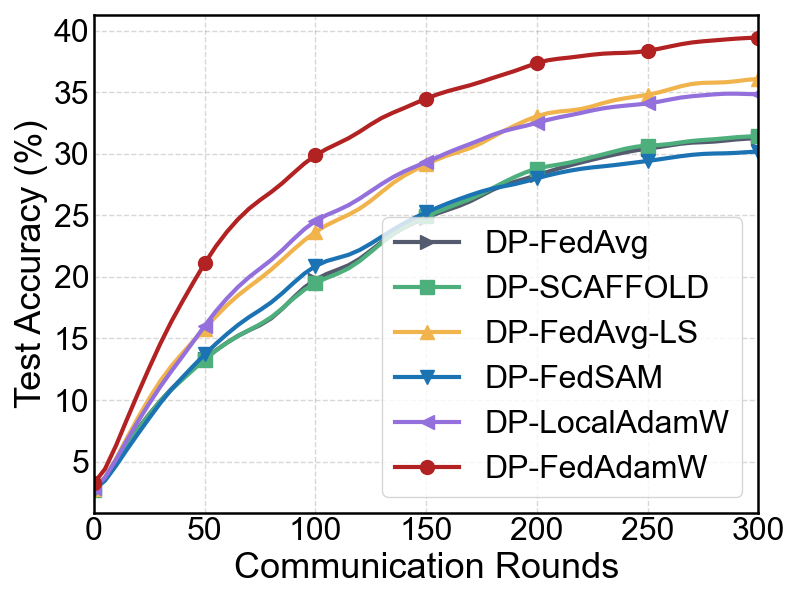}}
    \subcaptionbox{ResNet-18, $\alpha= 0.1$ \label{fig:sub2}}{\includegraphics[width=0.24\textwidth]{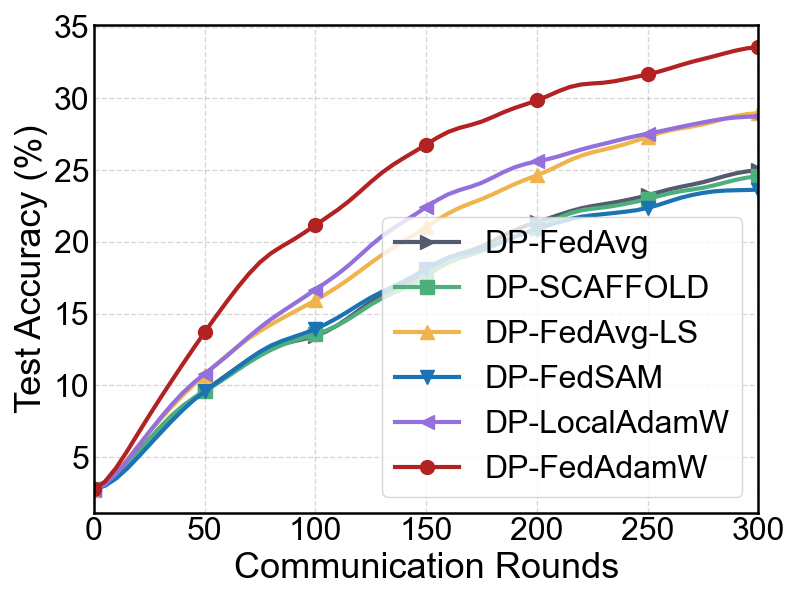}}
    \subcaptionbox{Swin-Tiny, $\alpha= 0.6$ \label{fig:sub3}}{\includegraphics[width=0.24\textwidth]{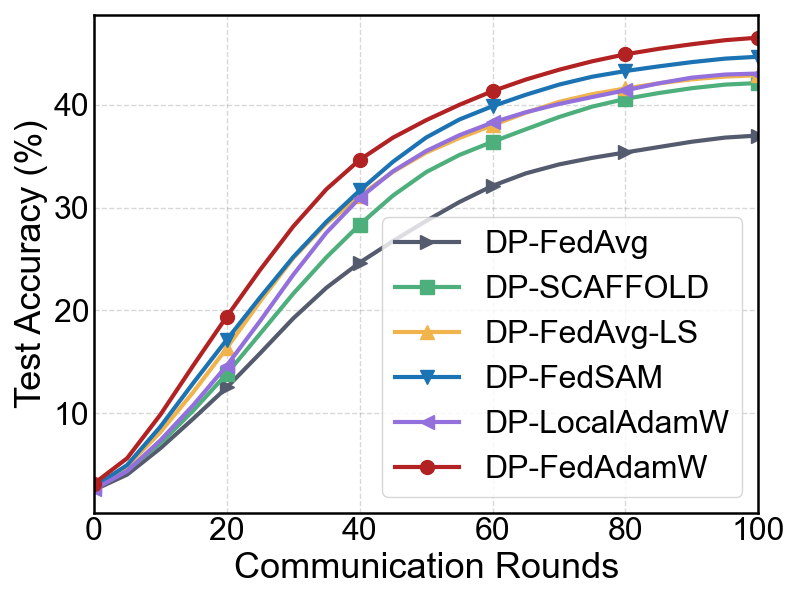}}
    \subcaptionbox{Swin-Tiny, $\alpha= 0.1$ \label{fig:sub4}}{\includegraphics[width=0.24\textwidth]{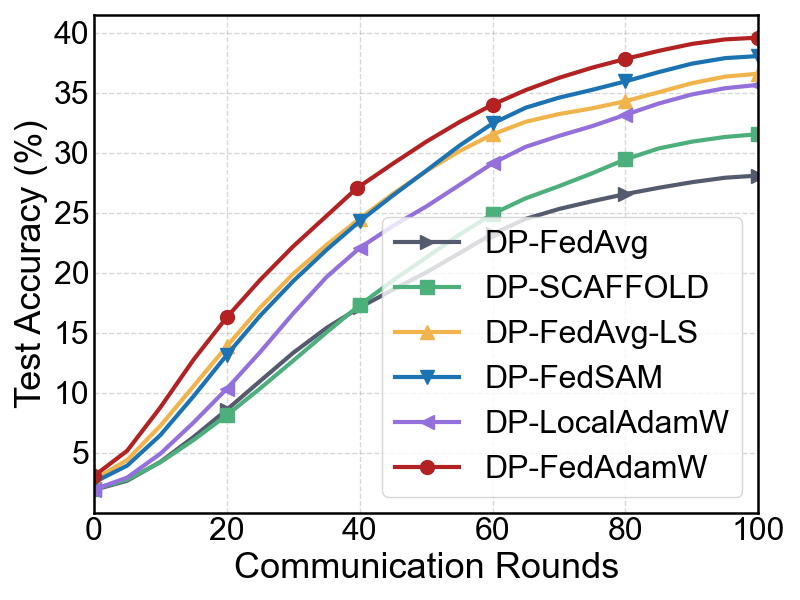}}
    \vspace{-2mm}
    \caption{Test accuracy (\%) on CIFAR-100 using ResNet-18 and Swin-Tiny under the Dirichlet $\alpha=0.6$ and $\alpha=0.1$ settings.}
    \label{fig:cifar100}
\end{figure*}

\begin{table}[tb]
\small
  \centering
  \caption{Ablation on the main components of DP-FedAdamW on Swin-Base.
  We report test accuracy (\%) on CIFAR-100 and Tiny-ImageNet under
  non-IID setting ($\alpha=0.1$). ``Agg'' denotes block-wise second-moment
  aggregation, ``BC'' the DP bias correction, and ``Align'' the
  local--global alignment term.}
  \vspace{-2mm}
  \label{tab:ablation_components}
  \renewcommand{\arraystretch}{0.9}
  \begin{tabular}{lcc}
    \toprule
    \textbf{Method} & \textbf{CIFAR-100} & \textbf{Tiny-ImageNet} \\
    \midrule
    DP-LocalAdamW
      & $65.28$ & $46.52$ \\
    w/o Agg
      & $66.41$ & $47.35$ \\
    w/o BC
      & $67.02$ & $48.11$ \\
    w/o Align
      & $66.37$ & $47.82$ \\
      \rowcolor{gray!10}
    DP-FedAdamW (full)
      & $\color{red}{67.42}$& $\color{red}{50.73}$ \\
    \bottomrule
  \end{tabular}
\end{table}
\vspace{-1mm}
\begin{table}[tb]
\small
  \centering
  \caption{\small Ablation study of moment aggregation strategies of DP-LocalAdamW on CIFAR-100 with Swin-Base under $\alpha{=}0.1$.}
  \vspace{-2mm}
  \label{tab:avg}
  \setlength{\tabcolsep}{2pt}
  \renewcommand{\arraystretch}{0.9}
  \begin{tabular}{lccc}
    \toprule
    \textbf{Strategy} & \textbf{CIFAR-100} & \textbf{Tiny-ImageNet} & \textbf{Comm(↑)} \\
    \midrule
    NoAgg      & $66.21_{\pm0.11}$ & $49.12_{\pm0.12}$ & $5.7$M \\

    Agg-$\boldsymbol{v}$     & $67.52_{\pm0.12}$ & $50.88_{\pm0.12}$ & $11.4$M \\    
    \rowcolor{gray!10}
    Agg-\texttt{mean}-$\boldsymbol{v}$    & $\color{red}{67.45}_{\pm0.10}$ & $\color{red}{50.76}_{\pm0.10}$ & 5.7M \\     
    \bottomrule
  \end{tabular}
\end{table}

\vspace{-1mm}
\begin{table}[h]
\small
  \centering
  \caption{\small Impact of $\gamma$ using Swin-Base ($\alpha{=}0.1$).}
  \vspace{-2mm}
  \label{tab:alpha_ablation}
  \setlength{\tabcolsep}{3pt}
  \renewcommand{\arraystretch}{0.9}
  \begin{tabular}{l|ccccc}
    \toprule
    $\gamma$ & $0.00$ & $0.25$ & \textbf{$\color{red}0.50$} & $0.75$ & $1.00$ \\
    \midrule
    CIFAR-100 (\%)      & $66.37$ & $67.01$ & \textbf{$\color{red}67.45$} & $67.10$ & $66.52$ \\
    Tiny-ImageNet (\%)  & $47.82$ & $49.36$ & \textbf{$\color{red}50.76$} & $50.21$ & $49.18$\\
    \bottomrule
  \end{tabular}
  \vspace{-2mm}
\end{table}

\begin{table}[h]
\small
  \centering
  \caption{Comparison with DP variants of federated adaptive optimizers under the same DP settings (Swin-Base, $\alpha{=}0.1$, $\sigma{=}1$).}
  \vspace{-2mm}
  \label{tab:ablation_components}
  \renewcommand{\arraystretch}{0.9}
  \begin{tabular}{lcc}
    \toprule
    \textbf{Method} & \textbf{CIFAR-100} & \textbf{Tiny-ImageNet} \\
    \midrule
    DP-FedOpt(Adam)~\cite{reddiadaptive}
      & ${64.54}_{\pm0.56}$ & ${44.81}_{\pm0.62}$ \\
    DP-FAFED~\cite{wu2023faster}
      & $65.75_{\pm0.54}$ & $46.10_{\pm0.67}$ \\
    DP-FedLADA~\cite{sun2023efficient}
      & $66.03_{\pm0.61}$ & $47.85_{\pm0.57}$ \\
      \rowcolor{gray!10}
    DP-FedAdamW (ours)
      &  $\color{red}{67.45}_{\pm0.26}$ & $\color{red}{50.76}_{\pm0.51}$ \\
    \bottomrule
  \end{tabular}
  \vspace{-2mm}
\end{table}

\section{Experiments}

\subsection{Experimental setup}
\noindent
\textbf{DPFL Algorithms.} We conduct experiments on six DPFL
algorithms: DP-FedAvg~\cite{noble2022differentially}, DP-SCAFFOLD~\cite{noble2022differentially},
DP-FedAvg-LS~\cite{liang2024differentially}, DP-FedSAM~\cite{shi2023make}, DP-LocalAdamW, and our DP-FedAdamW. In all performance evaluation experiments, the DP noise multiplier $\sigma$ is fixed. Under the same clipping norm and communication rounds, DP-FedSAM requires roughly twice the privacy budget of other methods, as it injects independent DP noise into two gradient evaluations at each local step. All methods are implemented using PyTorch and executed on NVIDIA RTX 4090 GPUs.

\noindent
\textbf{Datasets and data partition.} We evaluate DP-FedAdamW on seven datasets covering both vision and language tasks. For image classification,  we adopt CIFAR-10, CIFAR-100~\cite{krizhevsky2009learning}, and Tiny-ImageNet~\cite{tinyimagenet}. For natural language understanding, we consider four GLUE~\cite{wang2018glue} benchmarks: SST-2, QQP, QNLI, and MNLI. We simulate non-IID client data using a Dirichlet partition~\cite{hsu2019measuring}, where each client’s distribution is sampled from Dir($\alpha$) with $\alpha \in \{0.1, 0.6, 0.8\}$.

\noindent
\textbf{Model architectures.} We study three representative architecture families. (1) convolutional neural networks(CNNs): GNResNet-18 evaluated on CIFAR-10/100. (2) vision transformers: ViT-Base for CIFAR-10/100 and Swin-Tiny/Base for CIFAR-10/100 and Tiny-ImageNet. (3) language transformer: RoBERTa-Base assessed on the GLUE.

\noindent
\textbf{Configurations.}
For DP-FedAvg, DP-SCAFFOLD, DP-FedAvg-LS, and DP-FedSAM, the learning rate $\eta$ is selected from $\{10^{-2},\,3\times 10^{-2},\,5\times 10^{-2},\, 10^{-1},\,3\times 10^{-1}\}$ with a weight decay of $0.001$. For DP-LocalAdamW, DP-FedAdamW, $\eta$ is selected from $\{10^{-4},\,3\times 10^{-4},\,5\times 10^{-4},\,8\times 10^{-4},\, 10^{-3}\}$ with weight decay $0.01$ or $0.001$, $\beta_1 = 0.9$, $\beta_2 = 0.999$. We apply cosine learning
rate decay, and set DP-FedAdamW to $\gamma{=}0.5$, weight decay $\lambda{=}0.01$. For ResNet-18, ronds $T{=}300$, batch size $50$, local step $K{=}50$. For Transformers, ronds $T{=}100$, batch size $16$, local step $K{=}20$.
The noise multiplier is $\sigma{=}1$. The privacy failure probability is set to $\delta=10^{-5}$ for CIFAR10/100 and $\delta=10^{-6}$ for Tiny-ImageNet. Please refer to Appendix~\ref{app: experimental setup}
for more details of the experimental settings. All results are averaged
over 5 runs with std reported.

\subsection{Performance on vision datasets}

\noindent
\textbf{ResNet-18 on CIFAR-10/100.} 
Table~\ref{tab:resnet18_vitbase} summarizes the results on CIFAR-10/100 under different heterogeneity levels ($\alpha{=}0.6$ and $\alpha{=}0.1$). DP-FedAdamW achieves the best performance across all settings, reaching $74.81\%$ and $64.65\%$ on CIFAR-10, and $39.56\%$ and $33.55\%$ on CIFAR-100. When the data heterogeneity is severe ($\alpha{=}0.1$), DP-FedAdamW surpasses the SOTA baselines (DP-FedAvg-LS for CIFAR-10 and DP-LocalAdamW for CIFAR-100) by $\mathbf{3.81\%}$ and $\mathbf{4.85\%}$, respectively.

\noindent
\textbf{ViT-Base on CIFAR-10/100.}
We fine-tune ViT-Base on CIFAR-10/100, and the results are reported in Table~\ref{tab:resnet18_vitbase} under $\alpha{=}0.6$ and $\alpha{=}0.1$. DP-FedAdamW consistently achieves the best performance. On CIFAR-100 with $\alpha{=}0.1$, it outperforms the strongest baseline DP-FedSAM by $\mathbf{1.29\%}$. Recall that, under our DP configuration, DP-FedSAM is trained with a much weaker privacy guarantee (it uses twice the privacy budget of the other methods).

\noindent
\textbf{Swin-Tiny/Base on CIFAR-10/100/Tiny-ImageNet.}
We fine-tune Swin-Tiny and Swin-Base on CIFAR-10, CIFAR-100, and Tiny-ImageNet, and report the results in Tables~\ref{tab:swin_cifar100_tiny},~\ref{tab:swin_cifar10}. DP-FedAdamW consistently attains the highest accuracy across all settings. For example, on CIFAR-100 with Swin-Tiny at $\alpha{=}0.6$, it reaches $46.77\%$, outperforming the strongest baseline DP-FedSAM by $\mathbf{2.58\%}$. on Tiny-ImageNet with Swin-Base at $\alpha{=}0.1$, it achieves $50.85\%$, exceeding DP-FedSAM by $\mathbf{2.77\%}$.~\Cref{fig:cifar100} (c,d) shows that DP-FedAdamW achieves advantage on Swin-Tiny.

\vspace{-2mm}
\subsection{Performance on language datasets}
\noindent\textbf{RoBERTa-Base on GLUE.}
We further evaluate our method on four GLUE tasks (SST-2, QQP, QNLI, MNLI) by fine-tuning RoBERTa-Base under $\alpha{=}0.8$, and report the results in Table~\ref{tab:roberta_base}. DP-FedAdamW attains the best accuracy on all tasks, reaching $91.17\%$, $83.34\%$, $86.33\%$, and $78.68\%$, respectively. on the MNLI, DP-FedAdamW outperforms the baseline DP-LocalAdamW by $\mathbf{3.48\%}$, showing that the benefits of DP-FedAdamW extend from vision benchmarks to language understanding in the DPFL setting.

\subsection{Discussion for DP with AdamW in FL}

\noindent
\textbf{Performance under different privacy budgets $\varepsilon$.}
Table~\ref{tab:swin_base_epsilon} reports the results of Swin-Base on CIFAR-10, CIFAR-100, and Tiny-ImageNet under $\varepsilon{=}1,2,3$ with $\alpha{=}0.1$. As expected, all methods benefit from larger $\varepsilon$ (weaker privacy). Notably, the performance gap between DP-FedAdamW and other methods is most pronounced when $\varepsilon$ is small, i.e., under stronger privacy constraints. For instance, on CIFAR-10, DP-FedAdamW achieves $77.50\%$ at $\varepsilon{=}1$, surpassing the best baseline by $\mathbf{5.93\%}$. These results show that DP-FedAdamW achieves a superior utility–privacy trade-off.

\noindent
\textbf{Noise multiplier $\sigma$.} Results are provided in Appendix~\ref{app:Noise Multiplier Results}.

\subsection{Ablation study}



\noindent

\noindent\textbf{Effect of each component.}
DP-FedAdamW differs from DP-LocalAdamW in three aspects:
(i) block-wise second-moment aggregation (Agg);
(ii) DP bias correction on the second-moment (BC);
(iii) local--global alignment term (Align).
We construct several variants by removing one component at a time, while keeping all other settings fixed.

Table~\ref{tab:ablation_components} shows that each component yields a gain, and their combination achieves the best performance. Bias
correction (BC) improves robustness to DP noise, the local--global
alignment stabilizes optimization under non-IID and DP, and
block-wise aggregation provides an additional improvement with
negligible communication overhead.

\noindent\textbf{Impact of aggregation strategy.} Table \ref{tab:avg} shows that our block-wise averaging strategy, Agg-\texttt{mean}-$\boldsymbol{v}$, achieves the best balance between accuracy and communication cost. Agg-\texttt{mean}-$\boldsymbol{v}$ achieves comparable gains with only $\mathcal{O}(B)$ communication, where $B$ is the number of blocks, demonstrating its scalability and ability to stabilize updates even under strict communication budgets.

\noindent\textbf{Effect of the alignment coefficient $\gamma$.} 
Table~\ref{tab:alpha_ablation} varies $\gamma$ on Swin-Base ($\alpha{=}0.1$). Increasing $\gamma$ from $0.0$ to $0.5$ improves accuracy on CIFAR-100 and Tiny-ImageNet, 
indicating that a moderate alignment term effectively counteracts client drift and stabilizes optimization under DP noise. Further increasing $\gamma$ beyond $0.5$ leads to a slight degradation, as overly strong alignment over-constrains local updates. We therefore set $\gamma{=}0.5$ as the default.

\noindent\textbf{Comparison with DP variants of federated adaptive optimizers.} As shown in Table~\ref{tab:ablation_components}, on Swin-Base with Dirichlet $\alpha{=}0.1$, DP noise $\sigma{=}1$, simply adding DP to existing federated adaptive optimizers yields limited gains. DP-FAFED and DP-FedLADA slightly outperform DP-FedOpt(Adam) on both CIFAR-100 and Tiny-ImageNet. Our DP-FedAdamW consistently achieves the best performance, with roughly $1.42\%$ on CIFAR-100 and $2.91\%$ on Tiny-ImageNet compared to the SOTA DP baseline.

\vspace{-1mm}
\section{Conclusion}

In this work, we revisited AdamW for differentially private federated learning and revealed that a naïve adoption leads to amplified second-moment variance, systematic DP-induced bias, and pronounced client drift under small-batch local training. Building on these insights, we introduced DP-FedAdamW, which integrates block-wise second-moment aggregation, explicit DP bias correction, and a local–global alignment mechanism, with non-convex convergence and sample-level privacy guarantees. Experiments on vision (CIFAR-10/100, Tiny-ImageNet with ResNet/ViT/Swin) and language (GLUE with RoBERTa) benchmarks demonstrate DP-FedAdamW consistently surpasses SOTA DPFL baselines under the same privacy budgets, especially in highly non-IID regimes.
\newpage
{
    \small
    \bibliographystyle{ieeenat_fullname}
    \bibliography{main}
}

\onecolumn
\setcounter{page}{1}
\thispagestyle{empty}
\begin{center}
  {\Large\bfseries DP-FedAdamW: An Efficient Optimizer for Differentially Private Federated \\[2mm] Large Models}\\[4mm]
  {\Large Supplementary Material}
\end{center}

\section{More Implementation Detail}
\label{app: experiments}

\subsection{More Results}

\noindent
\textbf{Swin-Tiny/Base on CIFAR-10.} Table~\ref{tab:swin_cifar10} reports the averaged test accuracy on CIFAR-10 for six different DPFL methods evaluated on Swin-Tiny and Swin-Base, under two data heterogeneity levels (Dirichlet $\alpha{=}0.6$ and $\alpha{=}0.1$). Overall, DP-FedAdamW consistently achieves the best performance, while Swin-Base is markedly more accurate than Swin-Tiny and all methods degrade when the data becomes more heterogeneous (smaller $\alpha$). For Swin-Base with $\alpha{=}0.1$, DP-FedAdamW attains \textbf{$90.68\%$} test accuracy, outperforming DP-FedSAM ($89.82\%$) and DP-FedAvg ($88.23\%$), even though DP-FedSAM needs roughly twice the privacy budget.

\begin{table}[h]
	\vspace{-1mm}
	\centering
    \small
	\setlength{\tabcolsep}{1.3pt}
        \caption{Averaged test accuracy (\%) comparison on CIFAR-10. Assuming DP-FedAvg is $(\varepsilon,\delta)$-DP, all other methods share this budget except DP-FedSAM, which is approximately $(2\varepsilon,\delta)$-DP.}
	\renewcommand{\arraystretch}{0.8}
	\begin{tabular}{lccccc}
		\midrule[0.8pt]
		\multirow{2}{*}{\textbf{Method}} & \multicolumn{2}{c}{\textbf{Swin-Tiny}} & \multicolumn{2}{c}{\textbf{Swin-Base}} & \multirow{2}{*}{\shortstack{\textbf{Privacy} \\ \textbf{Budget}}}\\
		\cmidrule(lr){2-3} \cmidrule(lr){4-5}
		   & $\alpha{=}0.6$ & $\alpha{=}0.1$   &$\alpha{=}0.6$ & $\alpha{=}0.1$ \\
		\midrule
		DP-FedAvg    & $82.93_{\pm0.53}$ & $78.24_{\pm0.59}$ & $89.56_{\pm0.47}$ & $88.23_{\pm0.44}$  & $(\varepsilon,\delta)$\\
        DP-SCAFFOLD      & $82.85_{\pm0.50}$ & $78.53_{\pm0.54}$ & $90.07_{\pm0.48}$ & $89.08_{\pm0.62}$ & $(\varepsilon,\delta)$\\
        DP-FedAvg-LS    & $83.94_{\pm0.62}$ & $79.67_{\pm0.68}$ & $90.97_{\pm0.52}$ & $90.14_{\pm0.56}$ & $(\varepsilon,\delta)$\\
        DP-FedSAM     & $84.05_{\pm0.51}$ & $80.36_{\pm0.57}$ & $91.25_{\pm0.54}$ & $89.82_{\pm0.60}$ & $(2\varepsilon,\delta)$\\
		DP-LocalAdamW  & $83.73_{\pm0.58}$ & $79.28_{\pm0.64}$ & $90.84_{\pm0.39}$ & $89.08_{\pm0.41}$  & $(\varepsilon,\delta)$\\
       \rowcolor{gray!10}
		DP-FedAdamW & $\color{red}{84.23}_{\pm0.32}$ & $\color{red}{81.23}_{\pm0.39}$ & $\color{red}{91.79}_{\pm0.34}$ & $\color{red}{90.68}_{\pm0.38}$ & $(\varepsilon,\delta)$\\
		\midrule[0.8pt]
	\end{tabular}
  \label{tab:swin_cifar10}
  \vspace{-2mm}
\end{table}

\noindent
\textbf{Noise multiplier results.} 
\label{app:Noise Multiplier Results}
To further analyze the impact of privacy noise on algorithm performance, we evaluate different noise
multipliers $\sigma \in \{0.5, 1.0, 1.5, 2.0\}$ on CIFAR-100 using the Swin-Base model under Dirichlet $\alpha{=}0.6$ and $\alpha{=}0.1$. As shown in Table~\ref{tab:noise_multiplier}, increasing the noise multiplier consistently degrades
test accuracy for all methods, while DP-FedAdamW maintains the best performance across all $\sigma$. Although $\sigma{=}0.5$ yields
the highest accuracy, it also consumes a much larger privacy budget. In contrast, setting $\sigma{=}1$ significantly strengthens the formal privacy guarantee while only causing a small drop in accuracy, whereas larger noise levels ($\sigma{\ge}1.5$) lead to a much more pronounced loss in utility. Therefore, we adopt $\sigma{=}1$ as the default noise multiplier in our main experiments, as it offers a favorable privacy–utility trade-off.

\begin{table}[h]
	\centering
    \small
	\setlength{\tabcolsep}{1.3pt}
    \caption{Test accuracy (\%) on CIFAR-100 with Swin-Base under different noise multipliers $\sigma$.}
	\vspace{-2mm}
	\renewcommand{\arraystretch}{0.8}
	\begin{tabular}{lcccccccc}
		\midrule[0.8pt]
        \multirow{2}{*}{\textbf{Method}} & \multicolumn{4}{c}{\textbf{$\alpha=0.6$}} & \multicolumn{4}{c}{\textbf{$\alpha=0.1$}} \\
		\cmidrule(lr){2-5} \cmidrule(lr){6-9}
		   & $\sigma{=}0.5$ &$\sigma{=}1$  &$\sigma{=}1.5$& $\sigma{=}2$ & $\sigma{=}0.5$ &$\sigma{=}1$  &$\sigma{=}1.5$& $\sigma{=}2$\\
		\midrule
		DP-FedAvg    & $67.23_{\pm0.33}$ & $65.61_{\pm0.27}$ & $63.78_{\pm0.34}$ & $60.22_{\pm0.64}$ & $62.49_{\pm0.26}$ & $60.77_{\pm0.32}$ & $57.85_{\pm0.48}$ & $54.34_{\pm0.79}$ \\
        DP-SCAFFOLD    & $70.41_{\pm0.25}$ & $68.03_{\pm0.55}$ & $64.64_{\pm0.48}$ & $63.14_{\pm0.58}$ & $66.72_{\pm0.38}$ & $65.34_{\pm0.49}$ & $61.73_{\pm0.44}$ & $58.23_{\pm0.66}$\\
        DP-FedAvg-LS    & $70.93_{\pm0.40}$ & $68.80_{\pm0.52}$ & $66.52_{\pm0.43}$ & $64.77_{\pm0.63}$ & $68.36_{\pm0.39}$ & $66.08_{\pm0.56}$ & $62.87_{\pm0.52}$ & $59.98_{\pm0.58}$\\
        DP-FedSAM    & $72.55_{\pm0.22}$ & $70.44_{\pm0.57}$ & $67.45_{\pm0.33}$ & $65.75_{\pm0.69}$  & $67.94_{\pm0.51}$ & $66.28_{\pm0.64}$  & $62.21_{\pm0.60}$ & $60.84_{\pm0.71}$\\
		DP-LocalAdamW  & $70.74_{\pm0.13}$ & $68.33_{\pm0.33}$ & $66.46_{\pm0.53}$ & $64.10_{\pm0.60}$ & $67.35_{\pm0.31}$ & $65.28_{\pm0.39}$  & $60.55_{\pm0.43}$ & $58.78_{\pm0.74}$\\
        \rowcolor{gray!10}
		DP-FedAdamW & $\color{red}{73.62}_{\pm0.36}$ & $\color{red}{71.58}_{\pm0.30}$ & $\color{red}{68.52}_{\pm0.45}$ & $\color{red}{66.79}_{\pm0.51}$ & $\color{red}{69.22}_{\pm0.37}$ & $\color{red}{67.45}_{\pm0.26}$ & $\color{red}{63.28}_{\pm0.42}$ & $\color{red}{61.05}_{\pm0.61}$\\

		\midrule[0.8pt]
	\end{tabular}
  \label{tab:noise_multiplier}
  \vspace{-2mm}
\end{table}

\subsection{Datasets}

\noindent
\textbf{Vision datasets.}
We evaluate our methods on three widely used image classification benchmarks: CIFAR-10, CIFAR-100, and Tiny-ImageNet in Table~\ref{tab:vision_datasets}. 
CIFAR-10 and CIFAR-100~\cite{krizhevsky2009learning} each contain $60{,}000$ natural images of size $32\times 32$, split into $50{,}000$ training and $10{,}000$ test examples. CIFAR-10 comprises 10 coarse-grained object categories, whereas CIFAR-100 has finer labeling with 100 fine-grained categories organized into 20 superclasses. Tiny-ImageNet~\cite{tinyimagenet} is a downsampled subset of ImageNet with 200 classes, where each class contains 500 training images and 50 validation images of size $64\times 64$.

\begin{table}[h]
\small
\centering
\caption{Summary of the vision datasets used in our experiments.}
\vspace{-2mm}
\label{tab:vision_datasets}
\renewcommand{\arraystretch}{0.9}
\begin{tabular}{lccc}
\toprule
\textbf{Dataset} & \textbf{\#Classes} & \textbf{Image Size} & \textbf{Train / Test} \\
\midrule
CIFAR-10~\cite{krizhevsky2009learning}        & 10   & $32\times32$ & 50,000 / 10,000 \\
CIFAR-100~\cite{krizhevsky2009learning}       & 100  & $32\times32$ & 50,000 / 10,000 \\
Tiny-ImageNet~\cite{tinyimagenet}   & 200  & $64\times64$ & 100,000 / 10,000 \\
\bottomrule
\end{tabular}
\vspace{-2mm}
\end{table}

\noindent
\textbf{Language datasets.}
For natural language understanding, we experiment on four tasks from the GLUE~\cite{wang2018glue} benchmark: SST-2, QQP, QNLI, and MNLI (Table~\ref{tab:glue_datasets}. 
SST-2 is a binary sentiment classification task over movie reviews. 
QQP is a paraphrase identification task that determines whether two Quora questions express the same semantic intent. 
QNLI is a binary classification task that predicts whether a candidate sentence contains the answer to a given question. 
MNLI is a three-way natural language inference benchmark (entailment, contradiction, neutral) covering multiple genres of text. 
We follow the standard GLUE data splits and evaluation protocols.

\begin{table}[h]
\small
\centering
\caption{Summary of GLUE datasets used in our experiments.}
\vspace{-2mm}
\label{tab:glue_datasets}
\renewcommand{\arraystretch}{0.9}
\begin{tabular}{lccc}
\toprule
\textbf{Dataset} & \textbf{Task Type} & \textbf{\#Classes} & \textbf{Train/Test} \\
\midrule
SST-2 & Sentiment Classification (binary)            & 2 & 67,349 / 1,821 \\
QQP   & Duplicate Question Detection (binary)        & 2 & 363,846 / 390,965 \\
QNLI  & Question-Answer NLI (binary)                 & 2 & 104,743 / 5,463 \\
MNLI  & Natural Language Inference (entailment)      & 3 & 392,702 / 9,815 \\
\bottomrule
\end{tabular}
\vspace{-2mm}
\end{table}

\subsection{Models}
We study three representative architecture families, covering five specific models in total, as summarized in Table~\ref{tab:model}.

(1)~\textbf{ResNet-18.}
We adopt a ResNet-18 with Group Normalization (GN), where all Batch Normalization layers are replaced by GN to avoid dependence on batch statistics, which is more suitable for federated and differentially private training. 
To adapt CIFAR-10/100, we follow standard practices and modify the stem by replacing the original $7{\times}7$ convolution with a $3{\times}3$ kernel and removing the initial downsampling layers (i.e., the stride-2 convolution and max-pooling layer). 

(2)~\textbf{ViT-Base.} We use the standard Vision Transformer Base (ViT-Base) architecture, which splits an input image into non-overlapping $16{\times}16$ patches and linearly embeds them into a sequence of tokens with a prepended class token. 
The token sequence is processed by 12 Transformer encoder layers with multi-head self-attention and MLP blocks, and the final class token representation is fed to a linear classifier to produce the output logits.

(3)~\textbf{Swin-Tiny.} Swin-Tiny is a hierarchical Vision Transformer that uses shifted window self-attention. 
The model consists of four stages with depths $[2,2,6,2]$, embedding dimensions $[96,192,384,768]$, and attention heads $[3,6,12,24]$. 
Images are partitioned into $4{\times}4$ patches, and each stage applies window-based attention with window size $7$ and an MLP ratio of $4$, using LayerNorm and relative positional bias. DropPath regularization with depth-scaled rates is used, and the final features are aggregated by global pooling and passed to a linear classifier.

(4)~\textbf{Swin-Base.} Swin-Base is a larger variant of Swin-Tiny, sharing the same hierarchical shifted-window design. 
It is configured with four stages of depths $[2,2,18,2]$, embedding dimensions $[128,256,512,1024]$, and attention heads $[4,8,16,32]$, using a patch size of $4{\times}4$, window size~$7$, and MLP ratio~$4$. 
Similar to Swin-Tiny, the final stage features are pooled globally and fed into a linear classifier to generate predictions.

(5)~\textbf{RoBERTa-Base.}
We use the RoBERTa-Base architecture, which follows the standard Transformer encoder design with 12 layers, 12 attention heads, and a hidden size of 768. 
It employs byte-pair encoding (BPE) tokenization and is pretrained with a masked language modeling objective on large-scale corpora using dynamic masking. 
The final representation corresponding to the classification token is fed into a task-specific linear classifier.

\begin{table}[h]
  \small
  \centering
  \caption{Overview of the model architectures used in our experiments. 
  ``Depth'' denotes the total number of layers (blocks for ResNet/Swin, encoder layers for ViT/RoBERTa), 
  and ``Stages'' denotes the number of blocks per stage for hierarchical models.}
  \vspace{-2mm}
  \label{tab:model_arch}
  \renewcommand{\arraystretch}{1.0}
  \begin{tabular}{lccccc}
    \toprule
    \textbf{Model} & \textbf{Family} & \textbf{Depth} & \textbf{Stages} & \textbf{Width / Dim} & \textbf{Heads} \\
    \midrule
    ResNet-18 (GN) 
      & CNN 
      & 18 
      & $[2,2,2,2]$ 
      & $[64,128,256,512]$ 
      & -- \\
    ViT-Base
      & ViT Transformer
      & 12 
      & -- 
      & 768 hidden,\ 3072 MLP 
      & 12 / layer \\
    Swin-Tiny 
      & Swin Transformer
      & 12 
      & $[2,2,6,2]$ 
      & $[96,192,384,768]$ 
      & $[3,6,12,24]$ \\
    Swin-Base 
      & Swin Transformer
      & 24 
      & $[2,2,18,2]$ 
      & $[128,256,512,1024]$ 
      & $[4,8,16,32]$ \\
    RoBERTa-Base 
      & NLP Transformer
      & 12 
      & -- 
      & 768 hidden,\ 3072 FFN 
      & 12 / layer \\
    \bottomrule
  \end{tabular}
  \label{tab:model}
\end{table}

\begin{algorithm}[h]
\caption{DP-LocalAdamW}
\label{alg:DPLocalAdamW}
\begin{algorithmic}[1]
\STATE \textbf{Initial} initial model $\boldsymbol{\theta}^0$; rounds $T$; local steps $K$; step size $\eta$; weight decay $\lambda$; AdamW $(\beta_1,\beta_2,\epsilon)$; clip norm $C$; noise multiplier $\sigma$; $S$ clients per round

\FOR{$t=1$ \TO $T$}
  \FOR{selected clients $i=1,\dots,S$ in parallel}
    \STATE $\boldsymbol{\theta}_i^{t,0} \leftarrow \boldsymbol{\theta}^t$;\ \ $\boldsymbol{m}_i^{t,0} \leftarrow \boldsymbol{0}$;\ \ $\boldsymbol{v}_i^{t,0} \leftarrow \boldsymbol{0}$
    \FOR{$k=1$ \TO $K$}
      \STATE sample $\mathcal{D}_i^k \subset\mathcal{D}_i$ of size $\lfloor s R\rfloor$
      \FOR{each sample $j \in \mathcal{D}_i^k$}
        \STATE $g_{ij} \leftarrow \nabla f_i(\boldsymbol{\theta}_i^{t,k};\xi_{ij})$
        \STATE $\bar{g}_{ij} \leftarrow g_{ij} / \max(1, \|g_{ij}\|_2 / C)$
      \ENDFOR
      \STATE $\tilde g_i^{t,k} \leftarrow \frac{1}{s R} \sum_{j \in \mathcal{D}_i^k} \bar{g}_{i j}+\frac{ C}{s R}\mathcal{N}\!\left(0, \sigma^2 C^2 I\right)$ 
      \STATE $\boldsymbol{m}_i^{t,k} \leftarrow \beta_1 \boldsymbol{m}_i^{t,k-1} + (1-\beta_1)\,\tilde g_i^{t,k}$
      \STATE $\boldsymbol{v}_i^{t,k} \leftarrow \beta_2 \boldsymbol{v}_i^{t,k-1} + (1-\beta_2)\,\tilde g_i^{t,k} \odot \tilde g_i^{t,k}$
      \STATE $\hat{\boldsymbol{m}}_i^{t,k} \leftarrow \boldsymbol{m}_i^{t,k}/(1-\beta_1^{k})$
      \STATE $\hat{\boldsymbol{v}}_i^{t,k} \leftarrow \boldsymbol{v}_i^{t,k}/(1-\beta_2^{k})$
      \STATE $\boldsymbol{\theta}_i^{t,k+1} \leftarrow \boldsymbol{\theta}_i^{t,k} - \eta\,(\hat{\boldsymbol{m}}_i^{t,k}/(\sqrt{\hat{\boldsymbol{v}}_i^{t,k}} + \epsilon)-\lambda \boldsymbol{\theta}_i^{t,k})$
    \ENDFOR
    \STATE Clients send $(\boldsymbol{\theta}_i^{t,K} - \boldsymbol{\theta}_i^{t,0})$ \ to server
  \ENDFOR
  \STATE $\boldsymbol{\theta}^{t+1} \leftarrow \boldsymbol{\theta}^{t} + \frac{1}{S}\sum_{i=1}^S(\boldsymbol{\theta}_i^{t,K} - \boldsymbol{\theta}_i^{t,0})$ 
  \STATE Broadcast $\boldsymbol{\theta}^{t+1}$ to clients
\ENDFOR

\end{algorithmic}
\end{algorithm}

\subsection{Complete Configuration}
\label{app: experimental setup}

\noindent
\textbf{Federated setup.}
In all experiments, we fine-tune the models using parameter-efficient Low-Rank Adaptation (~\textbf{LoRA}).
We then detail the federated learning configurations adopted in our experiments.
Table~\ref{tab:fl_vision_dataset_setting} specifies the settings for all vision benchmarks, including the number of clients $N$, communication rounds $T$, local update steps $K$, client participation rate $l$, local batch size $sR$, DP noise multiplier $\sigma$, failure probability $\delta$, and LoRA hyperparameters $(r, \alpha_L)$.
Table~\ref{tab:fl_NLP_dataset_setting} reports the corresponding configuration for the GLUE benchmarks with RoBERTa-Base and LoRA, using the same federated parameters while additionally documenting the maximum sequence length and dropout rate.

\begin{table}[t]
  \small
  \centering
  \caption{Federated configurations for vision models across CIFAR-10/100 and Tiny-ImageNet. We report the number of clients $N$, communication rounds $T$, local steps $K$, client participation rate $l$, local batch size $sR$, DP noise multiplier $\sigma$, failure probability $\delta$, and LoRA hyperparameters $(r, \alpha_L)$ and dropout rate.}
  \vspace{-2mm}
  \renewcommand{\arraystretch}{1.0}
  \begin{tabular}{llcccccccccc}
    \toprule
    \textbf{Dataset} & \textbf{Model} 
      & \textbf{$N$} & \textbf{$T$} & \textbf{$K$} 
      & \textbf{$l$} & \textbf{$sR$} & \textbf{$\sigma$} & \textbf{$\delta$}& \textbf{$r$ (LoRA)} & \textbf{$\alpha_L$ (LoRA)} & \textbf{Dropout}\\
    \midrule
        \multirow{2}{*}{CIFAR-10} 
      & ResNet-18 (GN) 
      & 50 & 300 & 50 & 0.2 & 50 & 1.0 & $10^{-5}$  & 16  & 32 & 0.1\\
      & ViT-Base
      & 50 & 100 & 20 & 0.1 & 16 & 1.0 & $10^{-5}$ & 16  & 32 & 0.1 \\
      & Swin-Tiny/Base
      & 50 & 100 & 20 & 0.1 & 16 & 1.0 & $10^{-5}$ & 16  & 32 & 0.1 \\
      \midrule
    \multirow{2}{*}{CIFAR-100} 
      & ResNet-18 (GN) 
      & 50 & 300 & 50 & 0.2 & 50 & 1.0 & $10^{-5}$ & 16  & 32 & 0.1\\
      & ViT-Base
      & 50 & 100 & 20 & 0.1 & 16 & 1.0 & $10^{-5}$ & 16  & 32 & 0.1\\
      & Swin-Tiny/Base 
      & 50 & 100 & 20 & 0.1 & 16 & 1.0 & $10^{-5}$ & 16  & 32 & 0.1\\
    \midrule
    Tiny-ImageNet 
      & Swin-Tiny/Base 
      & 50 & 100 & 20 & 0.1 & 16 & 1.0 & $10^{-6}$ & 16  & 32 & 0.1\\
    \bottomrule
  \end{tabular}
  \label{tab:fl_vision_dataset_setting}
\end{table}

\begin{table}[t]
\small
\centering
\caption{Federated configurations for RoBERTa-Base with LoRA on GLUE tasks. 
We report the number of clients $N$, communication rounds $T$, local steps $K$, 
client participation rate $l$, local batch size $sR$, DP noise multiplier $\sigma$, 
failure probability $\delta$, and LoRA hyperparameters $(r, \alpha_L)$, as well as 
the maximum sequence length and dropout rate.}
\vspace{-2mm}
\renewcommand{\arraystretch}{0.9}
\begin{tabular}{llccccccccccc}
\toprule
\textbf{Dataset} & \textbf{Model} & \textbf{$N$} & \textbf{$T$} & \textbf{$K$} & \textbf{$l$} & \textbf{$sR$} & \textbf{$\sigma$} & \textbf{$\delta$}& \textbf{$r$ (LoRA)} & \textbf{$\alpha_L$ (LoRA)} & \textbf{Length\_max} & \textbf{Dropout} \\
\midrule
SST-2  &  \multirow{4}{*}{RoBERTa-Base} & 4 & 100 & 20 & 1.0 & 16 & 1.0 & $10^{-5}$ & 16 & 32 & 128 & 0.1 \\
QQP    &                                       & 4 & 100 & 20 & 1.0 & 16 & 1.0 & $10^{-6}$ & 16 & 32 & 128 & 0.1 \\
QNLI   &                                       & 4 & 100 & 20 & 1.0 & 16 & 1.0 & $10^{-5}$ & 16 & 32 & 128 & 0.1 \\
MNLI   &                                       & 4 & 100 & 20 & 1.0 & 16 & 1.0 & $10^{-6}$ & 16 & 32 & 128 & 0.1 \\
\bottomrule
\end{tabular}
\label{tab:fl_NLP_dataset_setting}
\vspace{-2mm}
\end{table}

\noindent
\textbf{Other setup.} DP-FedAvg, DP-SCAFFOLD, DP-FedAvg-LS, and DP-FedSAM adopt learning rates selected from \{${10^{-2}, 3{\times}10^{-2}, 5{\times}10^{-2}, 10^{-1}, 3{\times}10^{-1}}$\} with a fixed weight decay of $0.001$. For DP-FedSAM, the SAM perturbation $\rho$ is selected via grid search over $\{0.01,\,0.05,\,0.1,\,0.5\}$. For DP-LocalAdamW and DP-FedAdamW, the learning rate is chosen from \{${10^{-4}, 2{\times}10^{-4}, 3{\times}10^{-4}, 5{\times}10^{-4}, 8{\times}10^{-4}, 10^{-3}}$\}, combined with weight decay in $\{0.01, 0.001\}$ and standard AdamW parameters $(\beta_1{=}0.9,\beta_2{=}0.999)$. We employ cosine learning-rate decay throughout, and DP-FedAdamW additionally uses parameter $\gamma{=}0.5$ and weight decay $\lambda{=}0.01$. For all visual fine-tuning tasks, we initialize from official ImageNet-22K pre-trained weights to ensure consistency across methods. Across all tasks, the gradient clipping threshold is chosen by grid search over $\{0.05,\,0.1,\,0.2,\,0.3,\,0.5\}$, and we find that larger clipping values can trigger gradient explosion. To obtain a simple and robust configuration across all benchmarks, we therefore adopt a unified clipping threshold of $C{=}0.1$.

\begin{table}[h]
\small
\centering
\caption{Hyperparameter configuration of ResNet-18 (CIFAR-10/100) across different algorithms.}
\vspace{-2mm}
\renewcommand{\arraystretch}{0.9}
\begin{tabular}{lcccccc}
\toprule
\textbf{Method} & \textbf{Local Optimizer} & \textbf{$\eta$} & \textbf{$\beta_1$} & \textbf{$\beta_2$} & \textbf{Weight Decay} \\
\midrule
DP-FedAvg    &  SGD & $10^{-1}$ & - & - & 0.001 \\
DP-SCAFFOLD  &  SGD & $10^{-1}$ & - & - & 0.001 \\
DP-FedAvg-LS &  SGD & $10^{-1}$ & - & - & 0.001 \\
DP-FedSAM    &  SAM & $10^{-1}$ & - & - & 0.001 \\
DP-LocalAdamW&  AdamW & $3{\times}10^{-4}$ & 0.9 & 0.999 & 0.01 \\
DP-FedAdamW  &  AdamW & $3{\times}10^{-4}$ & 0.9 & 0.999 & 0.01 \\
\bottomrule
\end{tabular}
\label{tab:fl_NLP_dataset_setting}
\vspace{-2mm}
\end{table}

\section{DP-LocalAdamW Algorithm}

For completeness, we provide in Algorithm~\ref{alg:DPLocalAdamW} the full local training procedure of \textbf{DP-LocalAdamW}.

\section{Theoretical Analysis Details}

\subsection{Proof of Theorem 1 and Convergence Analysis}
\label{app:convergence_analysis}

\renewcommand{\theassumption}{A.\arabic{assumption}} 
\setcounter{assumption}{0}  

\begin{assumption}[Smoothness]
	\label{smoothness}
	(Smoothness) \textit{The non-convex $f_{i}$ is a $L$-smooth function for all $i\in[m]$, i.e., $\Vert\nabla f_{i}(\boldsymbol{\theta}_1)-\nabla f_{i}(\boldsymbol{\theta}_2)\Vert\leq L\Vert\boldsymbol{\theta}_1-\boldsymbol{\theta}_2\Vert$, for all $\boldsymbol{\theta}_1,\boldsymbol{\theta}_2\in\mathbb{R}^{d}$.}
\end{assumption}
\begin{assumption}[Bounded Stochastic Gradient]
	\label{bounded_stochastic_gradient_I}
    \textit{$\boldsymbol{g}_{i}^{t}=\nabla f_{i}(\boldsymbol{\theta}_{i}^{t}, \xi_i^{t})$ computed by using a sampled mini-batch data $\xi_i^{t}$ in the local client $i$ is an unbiased estimator of $\nabla f_{i}$ with bounded variance, i.e., $\mathbb{E}_{\xi_i^{t}}[\boldsymbol{g}_{i}^{t}]=\nabla f_{i}(\boldsymbol{\theta}_{i}^{t})$ and     $\mathbb{E}_{\xi_i^{t}}\Vert g_{i}^{t} - \nabla f_{i}(\boldsymbol{\theta}_{i}^{t})\Vert^{2} \leq \sigma_{l}^{2}$, for all $\boldsymbol{\theta}_{i}^{t}\in\mathbb{R}^{d}$.}
\end{assumption}
\begin{assumption}[Bounded Stochastic Gradient II]
	\label{bounded_stochastic_gradient_II}
	 \textit{Each element of stochastic gradient $\boldsymbol{g}_{i}^{t}$ is bounded, i.e., $\Vert\boldsymbol{g}_{i}^{t}\Vert_{\infty}=\Vert f_{i}(\boldsymbol{\theta}_{i}^{t},\xi_i^{t})\Vert_{\infty}\leq G_{g}$, for all $\boldsymbol{\theta}_{i}^{t}\in\mathbb{R}^{d}$ and any sampled mini-batch data $\xi_i^{t}$.}
\end{assumption}

\begin{assumption}[Bounded Heterogeneity]
\label{bounded_heterogeneity_appendix}
\textit{The gradient dissimilarity between clients is bounded: $\Vert\nabla f_{i}(\boldsymbol{\theta}_1)-\nabla f_{i}(\boldsymbol{\theta}_2)\Vert^{2}\leq\sigma_{g}^{2}$, for all $\boldsymbol{\theta}\in\mathbb{R}^{d}$.}
\end{assumption}

In this section, we give the theoretical analysis of our proposed DP-FedAdamW algorithm. Firstly we state some standard assumptions for the non-convex function $F$.

 We use the common gradient bound condition in our proof. 
Then we can upper bound $G_{\vartheta}$ as:
\begin{align*}
	G_{\vartheta} \
	&= \left\{\Vert\vartheta_{i}^{t,k}\Vert_{\infty}\right\}=\left\{\Vert\frac{1}{\sqrt{(1-\beta_1^{k})^2/(1-\beta_2^{t})\boldsymbol{v}_{i}^{t,k}}+\epsilon}\Vert_{\infty}\right\}\\
	&=\left\{\Vert\frac{1}{\sqrt{(1-\beta_1^{k})^2/(1-\beta_2^{t})(\beta_2  \boldsymbol{v}^{t, k-1}_i+\left(1-\beta_2\right) \boldsymbol{g}^{t,k}_i\cdot \boldsymbol{g}^{t,k}_i})+\epsilon}\Vert_{\infty}\right\}.
\end{align*}
$\sqrt{(1-\beta_1^{k})^2/(1-\beta_2^{t})(\beta_2  \boldsymbol{v}^{t, k-1}_i+\left(1-\beta_2\right) \boldsymbol{g}^{t,k}_i\cdot \boldsymbol{g}^{t,k}_i}$ is bounded as:
\begin{equation}
	\Vert\sqrt{(1-\beta_1^{k})^2/(1-\beta_2^{t})(\beta_2  \boldsymbol{v}^{t, k-1}_i+\left(1-\beta_2\right) \boldsymbol{g}^{t,k}_i\cdot \boldsymbol{g}^{t,k}_i}\Vert_{\infty}\leq G_g.
\end{equation}
Thus we bound $G_{\vartheta}$ as $\frac{1}{G_{g}}\leq G_{\vartheta} \leq\frac{1}{\epsilon}$.


\subsection{ Main Lemmas} 
\begin{lemma} \label{lem:bias-var} Suppose $\left\{X_1, \cdots, X_\tau\right\} \subset \mathbb{R}^d$ be random variables that are potentially dependent. If their marginal means and variances satisfy $\mathbb{E}\left[X_i\right]=\mu_i$ and $\mathbb{E}\left[\| X_i-\right.$ $\left.\mu_i \|^2\right] \leq \sigma^2$, then it holds that
	$$\mathbb{E}\left[\left\|\sum_{i=1}^\tau X_i\right\|^2\right] \leq\left\|\sum_{i=1}^\tau \mu_i\right\|^2+\tau^2 \sigma^2.
	$$
	If they are correlated in the Markov way such that $\mathbb{E}\left[X_i \mid X_{i-1}, \cdots X_1\right]=\mu_i$ and $\mathbb{E}\left[\left\|X_i-\mu_i\right\|^2 \mid\right.$ $\left.\mu_i\right] \leq \sigma^2$, i.e., the variables $\left\{X_i-\mu_i\right\}$ form a martingale. Then the following tighter bound holds:
	$$
	\mathbb{E}\left[\left\|\sum_{i=1}^\tau X_i\right\|^2\right] \leq 2 \mathbb{E}\left[\left\|\sum_{i=1}^\tau \mu_i\right\|^2\right]+2 \tau \sigma^2.
	$$
\end{lemma}

\begin{lemma} \label{lem:par_sample} Given vectors $v_1, \cdots, v_N \in \mathbb{R}^d$ and $\bar{v}=\frac{1}{N} \sum_{i=1}^N v_i$, if we sample $\mathcal{S} \subset\{1, \cdots, N\}$ uniformly randomly such that $|\mathcal{S}|=S$, then it holds that
	
	$$
	\mathbb{E}\left[\left\|\frac{1}{S} \sum_{i \in \mathcal{S}} v_i\right\|^2\right]=\|\bar{v}\|^2+\frac{N-S}{S(N-1)} \frac{1}{N} \sum_{i=1}^N\left\|v_i-\bar{v}\right\|^2 .
	$$
\end{lemma}

\begin{proof}
	Letting $\mathbb{I}\{i \in \mathcal{S}\}$ be the indicator for the event $i \in \mathcal{S}_r$, we prove this lemma by direct calculation as follows:
	$$
	\begin{aligned}
		\mathbb{E}\left[\left\|\frac{1}{S} \sum_{i \in \mathcal{S}} v_i\right\|^2\right] & =\mathbb{E}\left[\left\|\frac{1}{S} \sum_{i=1}^N v_i \mathbb{I}\{i \in \mathcal{S}\}\right\|^2\right] \\
		& =\frac{1}{S^2} \mathbb{E}\left[\left(\sum_i\left\|v_i\right\|^2 \mathbb{I}\{i \in \mathcal{S}\}+2 \sum_{i<j} v_i^{\top} v_j \mathbb{I}\{i, j \in \mathcal{S}\}\right)\right] \\
		& =\frac{1}{S N} \sum_{i=1}^N\left\|v_i\right\|^2+\frac{1}{S^2} \frac{S(S-1)}{N(N-1)} 2 \sum_{i<j} v_i^{\top} v_j \\
		& =\frac{1}{S N} \sum_{i=1}^N\left\|v_i\right\|^2+\frac{1}{S^2} \frac{S(S-1)}{N(N-1)}\left(\left\|\sum_{i=1}^N v_i\right\|^2-\sum_{i=1}^N\left\|v_i\right\|^2\right) \\
		& =\frac{N-S}{S(N-1)} \frac{1}{N} \sum_{i=1}^N\left\|v_i\right\|^2+\frac{N(S-1)}{S(N-1)}\|\bar{v}\|^2 \\
		& =\frac{N-S}{S(N-1)} \frac{1}{N} \sum_{i=1}^N\left\|v_i-\bar{v}\right\|^2+\|\bar{v}\|^2 .
	\end{aligned}
	$$
\end{proof}

\subsection{Basic Assumptions and Notations}
Let $\mathcal{F}^0=\emptyset$ and $\mathcal{F}_i^{t, k}:=\sigma\left(\left\{\boldsymbol{\theta}_i^{t, j}\right\}_{0 \leq j \leq k} \cup \mathcal{F}^t\right)$ and $\mathcal{F}^{t+1}:=\sigma\left(\cup_i \mathcal{F}_i^{t, K}\right)$ for all $t \geq 0$ where $\sigma(\cdot)$ indicates the $\sigma$-algebra. Let $\mathbb{E}_t[\cdot]:=\overline{\mathbb{E}}\left[\cdot \mid \mathcal{F}^t\right]$ be the expectation, conditioned on the filtration $\mathcal{F}^t$, with respect to the random variables $\left\{\mathcal{S}^t,\left\{\xi_i^{t, k}\right\}_{1 \leq i \leq N, 0 \leq k<K}\right\}$ in the $t$-th iteration. We also use $\mathbb{E}[\cdot]$ to denote the global expectation over all randomness in algorithms. Through out the proofs, we use $\sum_i$ to represent the sum over $i \in\{1, \ldots, N\}$, while $\sum_{i \in \mathcal{S}^t}$ denotes the sum over $i \in \mathcal{S}^t$. Similarly, we use $\sum_k$ to represent the sum of $k \in\{0, \ldots, K-1\}$. For all $t \geq 0$, we define the following auxiliary variables to facilitate proofs:

$$
\begin{aligned}
	\mathcal{E}_t & :=\mathbb{E}\left[\left\|\nabla f\left(\boldsymbol{\theta}^{t}\right)\odot\frac{1}{S K} \sum_{i \in S^t} \sum_{k=1}^K \vartheta_i^{t, k}-{\Delta}^{t+1}_G \right\|^2\right], \\
	U_t & :=\frac{1}{N K} \sum_i \sum_k \mathbb{E}\left[\left\|\boldsymbol{\theta}_i^{t, k}-\boldsymbol{\theta}^{t}\right\|\right]^2,\\
	\xi_i^{t, k} & :=\mathbb{E}\left[\boldsymbol{\theta}_i^{t, k+1}-\boldsymbol{\theta}_i^{t, k} \mid \mathcal{F}_i^{t, k}\right], \\
	\Xi_t & :=\frac{1}{N} \sum_{i=1}^N \mathbb{E}\left[\left\|\xi_i^{t, 0}\right\|^2\right]. \\
\end{aligned}
$$

 Throughout the Appendix, we let $\Delta:=f\left(\boldsymbol{\theta}^0\right)-f^{\star}, G_0:=\frac{1}{N} \sum_i\left\|\nabla f_i\left(\boldsymbol{\theta}^0\right)\right\|^2, \boldsymbol{\theta}^{-1}:=\boldsymbol{\theta}^0$ and $\mathcal{E}_{-1}:=$ $\mathbb{E}\left[\left\|\nabla f\left(\boldsymbol{\theta}^0\right)-g^0\right\|^2\right]$. We will use the following foundational lemma for all our algorithms.

\subsection{Proof of Theorem \ref{theorem_convergence_rate1}}
\begin{lemma} \label{lem:Fedavg_grad_err}
	Under Assumption \ref{smoothness} , if $\gamma L \leq \frac{1}{24}$, the following holds all $t \geq 0$ :	
	$$
\mathbb{E}\left[f\left(\boldsymbol{\theta}^{t+1}\right)\right] \leq \mathbb{E}\left[f\left(\boldsymbol{\theta}^{t}\right)\right]-\frac{11 \gamma}{24} \mathbb{E}\left[\left\|\nabla f\left(\boldsymbol{\theta}^{t}\right)\right\|^2\right]+\frac{13 \gamma}{24} \mathcal{E}_t+\frac{13 \gamma}{24}\frac{\sigma^2 G_{g}^2}{s^2 R^2}.
	$$
\end{lemma}

\begin{proof}
 Since $f$ is $L$-smooth, we have

$$
\begin{aligned}
	f\left(\boldsymbol{\theta}^{t+1}\right) & \leq f\left(\boldsymbol{\theta}^{t}\right)+\left\langle \nabla f\left(\boldsymbol{\theta}^{t}\right), \boldsymbol{\theta}^{t+1}-\boldsymbol{\theta}^{t}\right\rangle+\frac{L}{2}\left\|\boldsymbol{\theta}^{t+1}-\boldsymbol{\theta}^{t}\right\|^2 \\
	& \leq f\left(\boldsymbol{\theta}^{t}\right)+\gamma\left\langle \nabla f\left(\boldsymbol{\theta}^{t}\right) , {\Delta}^{t+1}_G \right\rangle+\frac{L\gamma^2}{2}\left\|{\Delta}^{t+1}_G \right\|^2 \\
	& =f\left(\boldsymbol{\theta}^{t}\right)-\gamma\left\|\nabla f\left(\boldsymbol{\theta}^{t}\right)\odot \frac{1}{S K} \sum_{i \in S^t} \sum_{k=1}^K \vartheta_i^{t, k}\right\|^2+\gamma\left\langle\nabla f\left(\boldsymbol{\theta}^{t}\right), \nabla f\left(\boldsymbol{\theta}^{t}\right)\odot \frac{1}{S K} \sum_{i \in S^t} \sum_{k=1}^K \vartheta_i^{t, k}-{\Delta}^{t+1}_G \right\rangle\\
   & +\frac{L \gamma^2}{2}\left\|{\Delta}^{t+1}_G \right\|^2 .
\end{aligned}
$$
Since $\boldsymbol{\theta}^{t+1}=\boldsymbol{\theta}^{t}-\gamma {\Delta}^{t+1}_G $, using Young's inequality, we further have
$$
\begin{aligned}
	&\mathbb{E} f\left(\boldsymbol{\theta}^{t+1}\right) \\
	& \leq f\left(\boldsymbol{\theta}^{t}\right)-(\gamma G_g^2-\frac{\gamma}{2})\left\|\nabla f\left(\boldsymbol{\theta}^{t}\right)\right\|^2+\frac{\gamma}{2}\left\|\nabla f\left(\boldsymbol{\theta}^{t}\right)\odot\frac{1}{S K} \sum_{i \in S^t} \sum_{k=1}^K \vartheta_i^{t, k}-{\Delta}^{t+1}_G \right\|^2\\
	&+L \gamma^2\left(\left\|\nabla f\left(\boldsymbol{\theta}^{t}\right)\odot\frac{1}{S K} \sum_{i \in S^t} \sum_{k=1}^K \vartheta_i^{t, k}\right\|^2+\left\|\nabla f\left(\boldsymbol{\theta}^{t}\right)\odot\frac{1}{S K} \sum_{i \in S^t} \sum_{k=1}^K \vartheta_i^{t, k}-{\Delta}^{t+1}_G \right\|^2\right) \\
	& \leq f\left(\boldsymbol{\theta}^{t}\right)-(\gamma G_g^2-\frac{\gamma}{2}-L\gamma^2 G_g^2)\left\|\nabla f\left(\boldsymbol{\theta}^{t}\right)\right\|^2+\frac{13 \gamma}{24}\left\|\nabla f\left(\boldsymbol{\theta}^{t}\right)\odot\frac{1}{S K} \sum_{i \in S^t} \sum_{k=1}^K \vartheta_i^{t, k}-{\Delta}^{t+1}_G \right\|^2+\frac{13 \gamma}{24}\frac{\sigma^2 G_{g}^2}{s^2 R^2}\\
	& \leq f\left(\boldsymbol{\theta}^{t}\right)-\frac{11 \gamma}{24}\left\|\nabla f\left(\boldsymbol{\theta}^{t}\right)\right\|^2+\frac{13 \gamma}{24}\left\|\nabla f\left(\boldsymbol{\theta}^{t}\right)\odot\frac{1}{S K} \sum_{i \in S^t} \sum_{k=1}^K \vartheta_i^{t, k}-{\Delta}^{t+1}_G \right\|^2+\frac{13 \gamma}{24}\frac{\sigma^2 G_{g}^2}{s^2 R^2},
\end{aligned}
$$
where the last inequality is due to $\gamma L \leq \frac{1}{24}$, $-(\gamma G_g^2-\frac{\gamma}{2}-L\gamma^2 G_g^2)\leq-\frac{11 \gamma}{24}$. Taking the global expectation, we finish the proof.
\end{proof}

\begin{lemma}\label{lem:Fedavg_client_drift}
	If $\gamma L \leq \frac{\gamma}{6}$, the following holds for $t \geq 1$ :
	
	$$
	\mathcal{E}_t \leq\left(1-\frac{8 \gamma}{9}\right) \mathcal{E}_{t-1}+\frac{4 \gamma^2 L^2}{\gamma} \mathbb{E}\left[\left\|\nabla f\left(\boldsymbol{\theta}^{t-1}\right)\right\|^2\right]+\frac{2 \gamma^2 \sigma_l^2}{S K \epsilon^2}+4 \frac{\gamma}{\epsilon^2} L^2 U_t.
	$$
	Additionally, it holds for $t=0$ that
	$$
	\mathcal{E}_0 \leq(1-\gamma) \mathcal{E}_{-1}+\frac{2 \gamma^2 \sigma_l^2}{S K}+4 \gamma L^2 U_0.
	$$
\end{lemma}

\begin{proof}
For $t>1$,
$$
\begin{aligned}
	\mathcal{E}_t= & \mathbb{E}\left[\left\|\frac{1}{S K} \sum_{i \in S^t} \sum_{k=1}^K \nabla f\left(\boldsymbol{\theta}^{t}\right)\odot\vartheta_i^{t, k}-{\Delta}^{t+1}_G \right\|^2\right] \\
	= & \mathbb{E}\left[\left\|(1-\gamma)\left(\frac{1}{S K} \sum_{i \in S^t} \sum_{k=1}^K \nabla f\left(\boldsymbol{\theta}^{t}\right)\odot\vartheta_i^{t, k}-{\Delta}^{t}_G \right)+\gamma\left(\frac{1}{S K} \sum_{i \in S^t} \sum_{k=1}^K \nabla f\left(\boldsymbol{\theta}^{t}\right)\odot\vartheta_i^{t, k}-\frac{1}{S K} \sum_{i \in S^t} \sum_{k=1}^K g_i^{t, k} \odot \vartheta_i^{t, k}\right)\right\|^2\right] \\
	\leq  & \mathbb{E}\left[\left\|(1-\gamma)\left(\frac{1}{S K} \sum_{i \in S^t} \sum_{k=1}^K \nabla f\left(\boldsymbol{\theta}^{t}\right)\odot\vartheta_i^{t, k}-{\Delta}^{t}_G \right)\right\|^2\right]+\gamma^2\frac{1}{\epsilon^2} \mathbb{E}\left[\left\|\nabla f\left(\boldsymbol{\theta}^{t}\right)-\frac{1}{S K} \sum_{i \in S^t} \sum_{k=1}^K g_i^{t, k}\right\|^2\right] \\
	& +2 \gamma \mathbb{E}\left[\left\langle(1-\gamma)\left(\frac{1}{S K} \sum_{i \in S^t} \sum_{k=1}^K \nabla f\left(\boldsymbol{\theta}^{t}\right)\odot\vartheta_i^{t, k}-{\Delta}^{t}_G \right), \frac{1}{S K} \sum_{i \in S^t} \sum_{k=1}^K \nabla f\left(\boldsymbol{\theta}^{t}\right)\odot\vartheta_i^{t, k}-\frac{1}{S K} \sum_{i \in S^t} \sum_{k=1}^K g_i^{t, k} \odot \vartheta_i^{t, k}\right\rangle\right] .
\end{aligned}
$$
Note that $\left\{\nabla F\left(\boldsymbol{\theta}_i^{t, k} ; \xi_i^{t, k}\right)\right\}_{0 \leq k<K}$ are sequentially correlated. Applying the AM-GM inequality and Lemma  \ref{lem:bias-var}, we have
$$
\mathcal{E}_t \leq\left(1+\frac{\gamma}{2}\right) \mathbb{E}\left[\left\|(1-\gamma)\left(\nabla f\left(\boldsymbol{\theta}^{t}\right)-{\Delta}^{t}_G \right)\right\|^2\right]+2\frac{\gamma}{\epsilon^2} L^2 U_t+2 \frac{\gamma^2}{\epsilon^2}\left(\frac{\sigma_l^2}{S K}+L^2 U_t\right).
$$
Using the AM-GM inequality again and Assumption \ref{smoothness}, we have
$$
\begin{aligned}
	\mathcal{E}_t & \leq(1-\gamma)^2\left(1+\frac{\gamma}{2}\right)\left[\left(1+\frac{\gamma}{2}\right) \mathcal{E}_{t-1}+\left(1+\frac{2}{\gamma}\right) L^2 \mathbb{E}\left[\left\|\boldsymbol{\theta}^{t}-\boldsymbol{\theta}^{t-1}\right\|^2\right]\right]+\frac{2 \gamma^2 \sigma_l^2}{S K\epsilon^2}+4 \frac{\gamma}{\epsilon^2} L^2 U_t \\
	& \leq(1-\gamma) \mathcal{E}_{t-1}+\frac{2}{\gamma} L^2 \mathbb{E}\left[\left\|\boldsymbol{\theta}^{t}-\boldsymbol{\theta}^{t-1}\right\|^2\right]+\frac{2 \gamma^2 \sigma_l^2}{\epsilon^2 S K}+4  \frac{\gamma}{\epsilon^2} L^2 U_t \\
	& \leq\left(1-\frac{8 \gamma}{9}\right) \mathcal{E}_{t-1}+4 \frac{\gamma^2 L^2}{\gamma} \mathbb{E}\left[\left\|\nabla f\left(\boldsymbol{\theta}^{t-1}\right)\right\|^2\right]+\frac{2 \gamma^2 \sigma_l^2}{\epsilon^2 S K}+4 \frac{\gamma}{\epsilon^2} L^2 U_t,
\end{aligned}
$$
where we plug in $\left\|\boldsymbol{\theta}^{t}-\boldsymbol{\theta}^{t-1}\right\|^2 \leq 2 \gamma^2\left(\left\|\nabla f\left(\boldsymbol{\theta}^{t-1}\right)\right\|^2+\left\|{\Delta}^{t}_G -\nabla f\left(\boldsymbol{\theta}^{t-1}\right)\right\|^2\right)$ and use $\gamma L \leq \frac{\gamma}{6}$ in the last inequality. Similarly for $t=0$,
$$
\begin{aligned}
	\mathcal{E}_0 & \leq\left(1+\frac{\gamma}{2}\right) \mathbb{E}\left[\left\|(1-\gamma)\left(\nabla f\left(\boldsymbol{\theta}^0\right)-g^0\right)\right\|^2\right]+2 \frac{\gamma}{\epsilon^2} L^2 U_0+2 \frac{\gamma^2}{\epsilon^2}\left(\frac{\sigma_l^2}{S K}+L^2 U_0\right) \\
	& \leq(1-\gamma) \mathcal{E}_{-1}+\frac{2 \gamma^2  \sigma_l^2}{S K\epsilon^2}+4 \frac{\gamma}{\epsilon^2} L^2 U_0.
\end{aligned}
$$
\end{proof}

\begin{lemma}
	If $\eta L K \leq \frac{1}{\gamma}$, the following holds for $t \geq 0$ :
	
	$$
	U_t \leq 2 e K^2 \Xi_t+K \eta^2 \gamma^2 \frac{1}{\epsilon^2} \sigma_l^2\left(1+2 K^3 L^2 \eta^2 \gamma^2\right).
	$$
\end{lemma}
\begin{proof}
Recall that $\xi_i^{t, k}:=\mathbb{E}\left[\boldsymbol{\theta}_i^{t, k+1}-\boldsymbol{\theta}_i^{t, k} \mid \mathcal{F}_i^{t, k}\right]=-\eta\left((1-\gamma) {\Delta}^{t}_G +\gamma \nabla f_i\left(\boldsymbol{\theta}_i^{t, k}\right)\odot \vartheta_i^{t, k}\right)$. Then we have
$$
\begin{aligned}
	\mathbb{E}\left[\left\|\xi_i^{t, j}-\xi_i^{t, j-1}\right\|^2\right] & \leq \frac{1}{\epsilon^2}\eta^2 L^2 \gamma^2 \mathbb{E}\left[\left\|\boldsymbol{\theta}_i^{t, j}-\boldsymbol{\theta}_i^{t, j-1}\right\|^2\right] \\
	& \leq \frac{1}{\epsilon^2}\eta^2 L^2 \gamma^2\left(\eta^2 \gamma^2 \sigma_l^2+\mathbb{E}\left[\left\|\xi_i^{t, j-1}\right\|^2\right)\right..
\end{aligned}
$$
For any $1 \leq j \leq k-1 \leq K-2$, using $\eta L \leq \frac{1}{\gamma K} \leq \frac{1}{\gamma(k+1)}$, we have
$$
\begin{aligned}
	\mathbb{E}\left[\left\|\xi_i^{t, j}\right\|^2\right] & \leq\left(1+\frac{1}{k}\right) \mathbb{E}\left[\left\|\xi_i^{t, j-1}\right\|^2\right]+(1+k) \mathbb{E}\left[\left\|\xi_i^{t, j}-\xi_i^{t, j-1}\right\|^2\right] \\
	& \leq\left(1+\frac{2}{k}\right) \mathbb{E}\left[\left\|\xi_i^{t, j-1}\right\|^2\right]+(k+1)\frac{1}{\epsilon^2} L^2 \eta^4 \gamma^4 \sigma_l^2 \\
	& \leq e^2 \mathbb{E}\left[\left\|\xi_i^{t, 0}\right\|^2\right]+4 \frac{1}{\epsilon^2}k^2 L^2 \eta^4 \gamma^4 \sigma_l^2.
\end{aligned}
$$
where the last inequality is by unrolling the recursive bound and using $\left(1+\frac{2}{k}\right)^k \leq e^2$. By Lemma \ref{lem:bias-var} , it holds that for $k \geq 2$,
$$
\begin{aligned}
	\mathbb{E}\left[\left\|\boldsymbol{\theta}_i^{t, k}-\boldsymbol{\theta}^{t}\right\|^2\right] & \leq 2 \mathbb{E}\left[\left\|\sum_{j=0}^{k-1} \xi_i^{t, j}\right\|^2\right]+2 \frac{1}{\epsilon^2}k \eta^2 \gamma^2 \sigma_l^2 \\
	& \leq 2 k \sum_{j=0}^{k-1} \mathbb{E}\left[\left\|\xi_i^{t, k}\right\|^2\right]+2 \frac{1}{\epsilon^2}k \eta^2 \gamma^2 \sigma_l^2 \\
	& \leq 2 e^2 k^2 \mathbb{E}\left[\left\|\xi_i^{t, 0}\right\|^2\right]+2 \frac{1}{\epsilon^2}k \eta^2 \gamma^2 \sigma_l^2\left(1+4 k^3 L^2 \eta^2 \gamma^2\right).
\end{aligned}
$$
This is also valid for $k=0,1$. Summing up over $i$ and $k$, we finish  the proof.
\end{proof}

\begin{lemma}\label{lem:Fedavg_grad_norm}
	If $288 e(\eta K L)^2\left((1-\gamma)^2+e(\gamma \gamma L R)^2\right) \leq 1$, then it holds for $t \geq 0$ that
	
	$$
	\sum_{t=0}^{R-1} \Xi_t \leq \frac{1}{72 e K^2 L^2} \sum_{t=-1}^{R-2}\left(\mathcal{E}_t+\mathbb{E}\left[\left\|\nabla f\left(\boldsymbol{\theta}^{t}\right)\right\|^2\right]\right)+2 \eta^2 \gamma^2\frac{1}{\epsilon^2} e R G_0.
	$$
\end{lemma}
\begin{proof}
 Note that $\xi_i^{t, 0}=-\eta\left((1-\gamma) {\Delta}^{t}_G +\gamma \nabla f_i\left(\boldsymbol{\theta}^{t}\right)\odot \vartheta_i^{t, k}\right)$,
$$
\frac{1}{N} \sum_{i=1}^N\left\|\xi_i^{t, 0}\right\|^2 \leq 2 \eta^2\left((1-\gamma)^2\left\|{\Delta}^{t}_G \right\|^2+\gamma^2 \frac{1}{N} \sum_{i=1}^N\left\|\nabla f_i\left(\boldsymbol{\theta}^{t}\right)\right\|^2\right).
$$
Using Young's inequality, we have for any $q>0$ that
$$
\begin{aligned}
	\mathbb{E}\left[\left\|\nabla f_i\left(\boldsymbol{\theta}^{t}\right)\right\|^2\right] & \leq(1+q) \mathbb{E}\left[\left\|\nabla f_i\left(\boldsymbol{\theta}^{t-1}\right)\right\|^2\right]+\left(1+q^{-1}\right) L^2 \mathbb{E}\left[\left\|\boldsymbol{\theta}^{t}-\boldsymbol{\theta}^{t-1}\right\|^2\right] \\
	& \leq(1+q) \mathbb{E}\left[\left\|\nabla f_i\left(\boldsymbol{\theta}^{t-1}\right)\right\|^2\right]+2\left(1+q^{-1}\right) \gamma^2 L^2\left(\mathcal{E}_{t-1}+\mathbb{E}\left[\left\|\nabla f\left(\boldsymbol{\theta}^{t-1}\right)\right\|^2\right]\right) \\
	& \leq(1+q)^t \mathbb{E}\left[\left\|\nabla f_i\left(\boldsymbol{\theta}^0\right)\right\|^2\right]+\frac{2}{q} \gamma^2 L^2 \sum_{j=0}^{t-1}\left(\mathcal{E}_j+\mathbb{E}\left[\left\|\nabla f\left(\boldsymbol{\theta}^j\right)\right\|^2\right)(1+q)^{t-j}\right..
\end{aligned}
$$
Take $q=\frac{1}{t}$ and we have

\begin{equation}\label{eqn:vninbvsdvds}
\mathbb{E}\left[\left\|\nabla f_i\left(\boldsymbol{\theta}^{t}\right)\right\|^2\right] \leq e \mathbb{E}\left[\left\|\nabla f_i\left(\boldsymbol{\theta}^0\right)\right\|^2\right]+2 e(t+1) \gamma^2 L^2 \sum_{j=0}^{t-1}\left(\mathcal{E}_j+\mathbb{E}\left[\left\|\nabla f\left(\boldsymbol{\theta}^j\right)\right\|^2\right)\right..
\end{equation}
Note that this inequality is valid for $t=0$. Therefore, using \eqref{eqn:vninbvsdvds}, we have
$$
\begin{aligned}
	\sum_{t=0}^{R-1} \Xi_t \leq & \sum_{t=0}^{R-1} 2 \eta^2 \mathbb{E}\left[(1-\gamma)^2\left\|{\Delta}^{t}_G \right\|^2+\gamma^2\frac{1}{\epsilon^2} \frac{1}{N} \sum_{i=1}^N\left\|\nabla f_i\left(\boldsymbol{\theta}^{t}\right)\right\|^2\right] \\
	\leq & \sum_{t=0}^{R-1} 2 \eta^2\left(2(1-\gamma)^2\left(\mathcal{E}_{t-1}+\mathbb{E}\left[\left\|\nabla f\left(\boldsymbol{\theta}^{t-1}\right)\right\|^2\right]\right)+\gamma^2\frac{1}{\epsilon^2} \frac{1}{N} \sum_{i=1}^N \mathbb{E}\left[\left\|\nabla f_i\left(\boldsymbol{\theta}^{t}\right)\right\|^2\right]\right) \\
	\leq & \sum_{t=0}^{R-1} 4 \eta^2(1-\gamma)^2\left(\mathcal{E}_{t-1}+\mathbb{E}\left[\left\|\nabla f\left(\boldsymbol{\theta}^{t-1}\right)\right\|^2\right]\right) \\
	& +2 \eta^2 \gamma^2\frac{1}{\epsilon^2} \sum_{t=0}^{R-1}\left(\frac{e}{N} \sum_{i=1}^N \mathbb{E}\left[\left\|\nabla f_i\left(\boldsymbol{\theta}^0\right)\right\|^2\right]+2 e(t+1)(\gamma L)^2 \sum_{j=0}^{t-1}\left(\mathcal{E}_j+\mathbb{E}\left[\left\|\nabla f\left(\boldsymbol{\theta}^j\right)\right\|^2\right]\right)\right) \\
	\leq & 4 \eta^2(1-\gamma)^2 \sum_{t=0}^{R-1}\left(\mathcal{E}_{t-1}+\mathbb{E}\left[\left\|\nabla f\left(\boldsymbol{\theta}^{t-1}\right)\right\|^2\right]\right) \\
	& +2 \eta^2 \gamma^2\frac{1}{\epsilon^2}\left(e R G_0+2 e(\gamma L R)^2 \sum_{t=0}^{R-2}\left(\mathcal{E}_t+\mathbb{E}\left[\left\|\nabla f\left(\boldsymbol{\theta}^{t}\right)\right\|^2\right]\right)\right).
\end{aligned}
$$
Rearranging the equation and applying the upper bound of $\eta$,  we finish the proof.
\end{proof}

\begin{theorem}[Convergence for non-convex functions]	
	Under Assumption \ref{smoothness} and \ref{bounded_stochastic_gradient_I} and \ref{bounded_stochastic_gradient_II} , if we take $g^0=0$,
	$$
	\begin{aligned}
		& \gamma=\min \left\{, \sqrt{\frac{S K L \Delta \epsilon^2}{\sigma_l^2 R}}\right\} \text { for any constant } c \in(0,1], \quad \gamma=\min \left\{\frac{1}{24 LG_g}, \frac{\gamma}{6 L}\right\}, \\
		& \eta K L \lesssim \min \left\{1, \frac{1}{\gamma \gamma L R},\left(\frac{L \Delta}{G_0 \gamma^3 R}\right)^{1 / 2}, \frac{1}{(\gamma N)^{1 / 2}}, \frac{1}{\left(\gamma^3 N K\right)^{1 / 4}}\right\}.
	\end{aligned}
	$$
	then DP-FedAdamW converges as
	
	$$
	\frac{1}{R} \sum_{t=0}^{R-1} \mathbb{E}\left[\left\|\nabla f\left(\boldsymbol{\theta}^{t}\right)\right\|^2\right] \lesssim \sqrt{\frac{L \Delta \sigma_l^2}{S K R}}+\frac{L \Delta}{R}+\frac{\sigma^2 G_{g}^2}{s^2 R^2} .
	$$

	Here $G_0:=\frac{1}{N} \sum_{i=1}^N\left\|\nabla f_i\left(\boldsymbol{\theta}^0\right)\right\|^2$.
\end{theorem}

\begin{proof}
 Combining Lemma \ref{lem:Fedavg_grad_err} and \ref{lem:Fedavg_client_drift}, we have

$$
\begin{aligned}
	\mathcal{E}_t \leq & \left(1-\frac{8 \gamma}{9}\right) \mathcal{E}_{t-1}+4 \frac{(\gamma L)^2}{\gamma} \mathbb{E}\left[\left\|\nabla f\left(\boldsymbol{\theta}^{t-1}\right)\right\|^2\right]+\frac{2 \gamma^2 \sigma_l^2}{S K \epsilon^2} \\
	& +4 \frac{\gamma}{\epsilon^2} L^2\left(2 e K^2 \Xi_t+K \eta^2 \gamma^2 \frac{1}{\epsilon^2} \sigma_l^2\left(1+2 K^3 L^2 \eta^2 \gamma^2\right)\right.,
\end{aligned}
$$
and
$$
\mathcal{E}_0 \leq(1-\gamma) \mathcal{E}_{-1}+\frac{2 \gamma^2 \sigma_l^2}{S K \epsilon^2}+4 \gamma \frac{1}{\epsilon^2} L^2\left(2 e K^2 \Xi_0+K \eta^2 \gamma^2 \frac{1}{\epsilon^2} \sigma_l^2\left(1+2 K^3 L^2 \eta^2 \gamma^2\right)\right).
$$
Summing over $t$ from 0 to $R-1$ and applying Lemma \ref{lem:Fedavg_grad_norm},
$$
\begin{aligned}
	\sum_{t=0}^{R-1} \mathcal{E}_t \leq & \left(1-\frac{8 \gamma}{9}\right) \sum_{t=-1}^{R-2} \mathcal{E}_t+4 \frac{(\gamma L)^2}{\gamma} \sum_{t=0}^{R-2} \mathbb{E}\left[\left\|\nabla f\left(\boldsymbol{\theta}^{t}\right)\right\|^2\right]+2 \frac{\gamma^2 \sigma_l^2}{S K \epsilon^2} R \\
	& +4 \gamma\frac{1}{\epsilon^2} L^2\left(2 e K^2 \sum_{t=0}^{R-1} \Xi_t+R K \eta^2 \gamma^2\frac{1}{\epsilon^2} \sigma_l^2\left(1+2 K^3 L^2 \eta^2 \gamma^2\right)\right) \\
	\leq & \left(1-\frac{7 \gamma}{9}\right) \sum_{t=-1}^{R-2} \mathcal{E}_t+\left(4 \frac{(\gamma L)^2}{\gamma}+\frac{\gamma}{9}\right) \sum_{t=-1}^{R-2} \mathbb{E}\left[\left\|\nabla f\left(\boldsymbol{\theta}^{t}\right)\right\|^2\right]+16 \gamma^3 \frac{1}{\epsilon^4}(e \eta K L)^2 R G_0 \\
	& +\frac{2 \gamma^2 \sigma_l^2}{S K \epsilon^2} R+4 \gamma^3\frac{1}{\epsilon^4}(\eta K L)^2\left(\frac{1}{K}+2(\eta K L \gamma)^2\right) \sigma_l^2 R \\
	\leq & \left(1-\frac{7 \gamma}{9}\right) \sum_{t=-1}^{R-2} \mathcal{E}_t+\frac{2 \gamma}{9} \sum_{t=-1}^{R-2} \mathbb{E}\left[\left\|\nabla f\left(\boldsymbol{\theta}^{t}\right)\right\|^2\right]+16 \gamma^3 \frac{1}{\epsilon^4}(e \eta K L)^2 R G_0+\frac{4 \gamma^2 \sigma_l^2}{S K \epsilon^2} R.
\end{aligned}
$$
Here in the last inequality we apply
$$
4 \gamma \frac{1}{\epsilon^4}(\eta K L)^2\left(\frac{1}{K}+2(\eta K L \gamma)^2\right) \leq \frac{2}{N K} \quad \text { and } \quad \gamma L \leq \frac{\gamma}{6} .
$$
Therefore,
$$
\sum_{t=0}^{R-1} \mathcal{E}_t \leq \frac{9}{7 \gamma} \mathcal{E}_{-1}+\frac{2}{7} \mathbb{E}\left[\sum_{t=-1}^{R-2}\left\|\nabla f\left(\boldsymbol{\theta}^{t}\right)\right\|^2\right]+\frac{144}{7}\frac{1}{\epsilon^4}(e \gamma \eta K L)^2 G_0 R+\frac{36 \gamma \sigma_l^2}{7 S K\epsilon^2} R .
$$
Combine this inequality with Lemma \ref{lem:Fedavg_grad_err} and we get
$$
\frac{1}{\gamma} \mathbb{E}\left[f\left(\boldsymbol{\theta}^{t}\right)-f\left(\boldsymbol{\theta}^0\right)\right] \leq-\frac{1}{7} \sum_{t=0}^{R-1} \mathbb{E}\left[\left\|\nabla f\left(\boldsymbol{\theta}^{t}\right)\right\|^2\right]+\frac{39}{56 \gamma} \mathcal{E}_{-1}+\frac{78}{7}\frac{1}{\epsilon^4}(e \gamma \eta K L)^2 G_0 R+\frac{39 \gamma \sigma_l^2}{14 S K \epsilon^2} R.
$$
Finally, noticing that $g^0=0$ implies $\mathcal{E}_{-1} \leq 2 L\left(f\left(\boldsymbol{\theta}^0\right)-f^*\right)=2 L \Delta$, we obtain
$$
\begin{aligned}
	\frac{1}{R} \sum_{t=0}^{R-1} \mathbb{E}\left[\left\|\nabla f\left(\boldsymbol{\theta}^{t}\right)\right\|^2\right] & \lesssim \frac{L \Delta}{\gamma L R}+\frac{\mathcal{E}_{-1}}{\gamma T}+(\gamma \eta K L)^2 \frac{1}{\epsilon^4} G_0+\frac{\gamma \sigma_l^2}{S K \epsilon^2}+\frac{\sigma^2 G_{g}^2}{s^2 T^2} \\
	& \lesssim \frac{L \Delta}{T}+\frac{L \Delta}{\gamma T}+\frac{\gamma \sigma_l^2}{S K \epsilon^2}+(\gamma \eta K L)^2 G_0 \frac{1}{\epsilon^4}+\frac{\sigma^2 G_{g}^2}{s^2 T^2}\\
	& \lesssim \frac{L \Delta}{T}+\sqrt{\frac{L \Delta \sigma_l^2}{S K T \epsilon^2}}+\frac{\sigma^2 G_{g}^2}{s^2 T^2}.
\end{aligned}
$$
\end{proof}


\end{document}